\documentclass[10pt]{article}
\usepackage[top=1in, bottom=1in, left=0.9in, right=0.9in]{geometry}

\usepackage{zhiyuz}

\title{\textbf{Discounted Adaptive Online Learning: Towards Better Regularization}}
\author{
  Zhiyu Zhang\\
  Harvard University\\
  \texttt{zhiyuz@seas.harvard.edu}\\
  \and
  David Bombara\\
  Harvard University\\
  \texttt{davidbombara@g.harvard.edu}\\
  \and
  Heng Yang\\
  Harvard University\\
  \texttt{hankyang@seas.harvard.edu}\\
}
\date{}

\begin{document}
\maketitle

\begin{abstract}
We study online learning in adversarial nonstationary environments. Since the future can be very different from the past, a critical challenge is to gracefully forget the history while new data comes in. To formalize this intuition, we revisit the discounted regret in online convex optimization, and propose an adaptive (i.e., instance optimal), FTRL-based algorithm that improves the widespread non-adaptive baseline -- gradient descent with a constant learning rate. From a practical perspective, this refines the classical idea of regularization in lifelong learning: we show that designing good regularizers can be guided by the principled theory of adaptive online optimization. 

Complementing this result, we also consider the (Gibbs and Cand\`es, 2021)-style online conformal prediction problem, where the goal is to sequentially predict the uncertainty sets of a black-box machine learning model. We show that the FTRL nature of our algorithm can simplify the conventional gradient-descent-based analysis, leading to instance-dependent performance guarantees. 
\end{abstract}

\section{Introduction}\label{section:introduction}

Online learning can be broadly defined as a sequential decision making problem, where each decision leverages the \emph{learned} knowledge from previous observations. However, while \emph{forgetting} is often thought as the opposite of \emph{learning}, the two concepts are actually coherent due to the \emph{distribution shifts} in practice. Think about deploying a drone in the wild: a common subtask is to learn its time-varying dynamics model on the fly, but when doing that, we have to exclude the obsolete data that possibly contradicts the current environment. This leads to a natural challenge for the resulting online learning problem: how do we handle the possible shortage of data after forgetting? 

One potential solution is to inject suitable \emph{inductive bias} into the algorithm, which is among the most important ideas in machine learning. Specifically, the inductive bias refers to our prior belief of the ground truth, before observing the data of interest. Regarding the drone example, physics provides general principles on the evolution of the nature, while modern foundation models can encode diverse ``world knowledge'' from large scale pre-training. The point is that even though forgetting reduces the amount of online data, one could still exploit such inductive bias to improve the learning performance.

Despite this natural intuition, algorithmically achieving it remains a nontrivial task. In particular, the associated online learning algorithm has two considerations to trade off (i.e., the inductive bias and the online data), and optimally balancing them requires going beyond simple heuristics. The present work studies this problem from a theoretical perspective, based on a \emph{discounted} variant of \emph{Online Convex Optimization} (OCO) \cite{zinkevich2003online,cesa2006prediction}. Our results emphasize the importance of \emph{regularization} in nonstationary online learning: 
\begin{quote}
On an algorithm that gradually forgets the online data, one could use regularization to inject the inductive bias, \emph{which the algorithm never forgets}. 
\end{quote}
Furthermore, as opposed to conventional regularization mechanisms like weight decay, we show that designing good regularizers can be guided by the theory of adaptive OCO. 

\subsection{Contribution}

This paper presents new results on two related topics: (Section~\ref{section:main}) nonstationary online learning, and (Section~\ref{section:conformal}) \emph{Online Conformal Prediction} (OCP) \citep{gibbs2021adaptive}.
\begin{itemize}
\item First, we consider the discounted OCO problem, formally introduced in Section~\ref{section:main}. When the discount factors are time-invariant, it is known that under standard assumptions on the complexity of the problem, the \emph{minimax optimal} discounted regret bound can be achieved by an extremely simple and widespread baseline -- \emph{Online Gradient Descent} with constant\footnote{Non-annealing and horizon-independent.} learning rate (denoted as constant-LR OGD). In this paper, we propose an \emph{adaptive} algorithm based on \emph{Follow the Regularized Leader} (FTRL) \citep{abernethy2008competing}, such that
\begin{itemize}[leftmargin=*,itemsep=0pt,topsep=0pt]
\item the minimax optimality of the OGD baseline is improved to \emph{instance optimality};
\item the discount factors can be arbitrarily time-varying; and
\item all assumptions beyond convexity are removed. 
\end{itemize}
More concretely, this is achieved by combining an undiscounted, adaptive, FTRL-based OCO algorithm \citep{zhang2024improving} with a simple \emph{rescaling trick} -- the latter can convert \emph{scale-free} \citep{orabona2018scale} upper bounds on the standard undiscounted regret to the discounted regret, which could be of independent interest.

In practice, compared to existing nonstationary OCO algorithms that minimize the \emph{dynamic regret} or the \emph{strongly adaptive regret}, our algorithm eschews the standard bi-level aggregation procedure \citep{hazan2009efficient,daniely2015strongly,zhang2018adaptive}, thus is computationally more efficient. Compared to ``vanilla'' approaches in lifelong learning based on constant-LR OGD \citep{abel2023definition}, our algorithm uses a nontrivial data-dependent regularizer to adaptively exploit the available inductive bias. Furthermore, the design of this regularizer follows from the principled theory of adaptive OCO, rather than heuristics. 

\item Next, we consider OCP, an online learning task with set-membership predictions. To combat the distribution shifts commonly found in practice, recent works \citep{gibbs2021adaptive,gibbs2022conformal,bhatnagar2023improved} applied nonstationary OCO algorithms (such as constant-LR OGD) to this setting. The twist is that besides appropriate regret bounds, one has to establish \emph{coverage} guarantees as well: the relative frequency of the prediction sets covering the true labels should converge to a pre-specified \emph{coverage rate}. 

In this setting, our discounted OCO algorithm leads to strong instance-dependent guarantees (with respect to the targeted coverage rate). Notably, since our algorithm is built on the FTRL framework rather than gradient descent, the associated coverage guarantee follows simply from the \emph{stability} of its iterates. This exemplifies the analytical strength of FTRL in OCP (over gradient descent), which, to our knowledge, has not been demonstrated in the literature. 
\end{itemize}

Finally, we complement the above theoretical results with OCP experiments (Section~\ref{section:experiment}). Code is available at \url{https://github.com/ComputationalRobotics/discounted-adaptive}.

\subsection{Related work}

This paper mainly explores the connection between two separate topics in online learning: discounting and adaptivity. We will also briefly compare our scope with the practical problem of lifelong / continual learning. Regarding the OCP problem (the second half of this paper), related works are discussed in Section~\ref{section:conformal}.

\paragraph{Discounting} Motivated by the intuition that ``the recent history is more important than the distant past'', the discounted regret has been studied by a series of works on nonstationary online learning \citep{cesa2006prediction,freund2008new,chernov2010prediction,kapralov2011prediction,cesa2012mirror,brown2019solving}. Most of them do not consider adaptivity, although the under-appreciated \citep{kapralov2011prediction} presented important early ideas. Recently, the discounted regret seems to lose its popularity in the online learning community to the dynamic regret \citep{zinkevich2003online} and the strongly adaptive regret \citep{daniely2015strongly}, which we survey in Appendix~\ref{section:more_related}. However, due to its computational efficiency, the idea of discounting is still prevalent in practice, as exemplified by the success of the \textsc{Adam} optimizer \citep{kingma2015adam}. 

Concurrent to this work, \cite{ahn2024understanding} presented a conversion from the discounted regret to the dynamic regret. \cite{jacobsen2024online} independently studied the dynamic and strongly adaptive regret of discounted algorithms, in the context of \emph{online linear regression}. 

\paragraph{Adaptivity} Focusing on stationary environments, adaptive OCO concerns going beyond the conventional worst case regret bounds \citep[Chapter~4 and 9]{orabona2023modern}. More than a decade of research effort culminated in a series of OCO algorithms that do not rely on any extra assumption beyond convexity \citep{cutkosky2019artificial,mhammedi2020lipschitz,chen2021impossible,jacobsen2022parameter,zhang2024improving,cutkosky2024fully}, which this paper builds on. Different from non-adaptive algorithms like OGD, such adaptivity (also called \emph{parameter-freeness}) crucially relies on various sophisticated forms of regularization \citep{mcmahan2014unconstrained,orabona2016coin,cutkosky2018black,zhang2022pde}. This naturally resonates with the crucial role of regularization in lifelong learning \citep{chaudhry2019efficient,de2021continual} (discussed next), but their quantitative connection has not been thoroughly studied in the literature. By considering adaptivity on the discounted regret, we aim to fill this gap. 

\paragraph{Practical lifelong learning} From a more practical perspective, lifelong learning (also called continual learning) refers to a highly relevant problem of ``learning a sequence of skills''. There are several different objectives under this topic \citep{lopez2017gradient}: regret minimization is closely related to \emph{forward transfer}, where the goal is to leverage existing knowledge to accelerate future learning. We do not analyze another important objective called \emph{backward transfer} or preventing \emph{catastrophic forgetting}, whose goal is to maintain existing knowledge while learning new skills. 

Different from typical lifelong learning settings, we also assume a major simplification -- convexity. This provides us the theoretical tractability to show that ``better regularization leads to better performance bounds'', but it cannot justify the \emph{stability} benefit of regularization, which is another important consideration in practice. Roughly speaking, the continual training of deep neural networks tends to collapse once the parameter leaves a ``benign region'', and regularization has been recognized as an effective way to mitigate this problem \citep{lyle2023understanding,press2023rdumb,sokar2023dormant}. The convexity assumption under OCO, however, ensures stability all the time.

\subsection{Notation}

Throughout this paper, $\norm{\cdot}$ denotes the Euclidean norm. $\Pi_\X(x)$ is the Euclidean projection of $x$ onto a closed convex set $\X$. The diameter of a set $\X$ is $\sup_{x,y\in\X}\norm{x-y}$. 

For two integers $a\leq b$, $[a:b]$ is the set of all integers $c$ such that $a\leq c\leq b$. The brackets are removed when on the subscript, denoting a tuple with indices in $[a:b]$. If $a>b$, then the product $\prod_{i=a}^b\lambda_i\defeq 1$. $0$ represents a zero vector whose dimension depends on the context. $\log$ means natural logarithm. 

We define the \emph{imaginary error function} as $\erfi(x)\defeq\int_0^x\exp(u^2)du$; this is scaled by $\sqrt{\pi}/2$ from the conventional definition, thus can also be queried from standard software packages like \textsc{SciPy} and \textsc{JAX}.

\section{Discounted adaptivity}\label{section:main}

\paragraph{Setting} The first half of this paper is about discounted \emph{Online Convex Optimization} (OCO), a two-person repeated game between a player we control and an adversarial environment. Different from the standard OCO problem \citep{zinkevich2003online}, there is also an \emph{expert} that sequentially selects the discount factors for the player. In each round, we consider the following interaction protocol. 
\begin{enumerate}
\item We (the player) make a prediction $x_t\in\X$ using past observations, where $\X\subset\R^d$ is closed and convex.
\item The environment picks a convex loss function $l_t:\X\rightarrow\R$, and reveal a subgradient $g_t\in\partial l_t(x_t)$ to the player. 
\item The expert picks a discount factor $\lambda_{t-1}\in(0,\infty)$,\footnote{The one round delay is due to our regret definition: the loss function $l_t$ is undiscounted when evaluated in the $t$-th round.} and reveal it to the player.
\item The environment can choose to terminate the game. If so, let $T$ be the total number of rounds. 
\end{enumerate}
At the end of the game, the environment can also choose 
any fixed prediction $u\in\X$, called a \emph{comparator}. Without knowing the environment and the expert beforehand, our (the player's) goal is to guarantee low \emph{discounted regret} with respect to the comparator, 
\begin{equation}\label{eq:discounted}
\reg^{\lambda_{1:T}}_T(l_{1:T},u)\defeq\sum_{t=1}^T\rpar{\prod_{i=t}^{T-1}\lambda_i}\spar{l_t(x_t)-l_t(u)}.
\end{equation}

We say an algorithm is \emph{minimax} or \emph{non-adaptive} if given an uncertainty set $\calS$, it upper-bounds the \emph{worst case regret}
\begin{equation*}
\sup_{(l_{1:T},u)\in\calS}\reg^{\lambda_{1:T}}_T(l_{1:T},u).
\end{equation*}
This paper aims to design \emph{adaptive} algorithms that directly upper-bound $\reg^{\lambda_{1:T}}_T(l_{1:T},u)$ itself by a function of the \emph{problem instance} (i.e., both the losses $l_{1:T}$ and the comparator $u$). 

\paragraph{Discounting as forgetting} Let us motivate the above setting a bit further. When $\lambda_t\equiv 1$, Eq.(\ref{eq:discounted}) recovers the standard undiscounted regret in OCO, from which we can further upper-bound the \emph{total loss} of the player, $\sum_{t=1}^Tl_t(x_t)$. Its per-round average $T^{-1}\sum_{t=1}^Tl_t(x_t)$ is often the performance we care about in practice. However, the effectiveness of this argument relies on the existence of a comparator $u$ with low loss $T^{-1}\sum_{t=1}^Tl_t(u)$. If this ``stationarity'' condition does not hold, then even when the regret bound is low, the player is not guaranteed to perform well.

This paper concerns the ``nonstationary'' environments violating the above argument. To make the problem tractable, the expert provides discount factors $\lambda_{1:T}$ to the player as side information, suggesting the usefulness of each observation for future predictions -- typically, how fast existing observations should be forgotten. For example, 
\begin{itemize}
\item With $\lambda_t\equiv \lambda<1$, the weight of past losses gradually decays in Eq.(\ref{eq:discounted}), therefore the corresponding algorithm is motivated to forget accordingly, matching the intuition developed in Section~\ref{section:introduction}. 
\item A more extreme case is when $\lambda_t$ takes value in $\{0,1\}$. Then, the problem reduces to a restarting variant of standard OCO, where each restart (as an extreme form of forgetting) is triggered by $\lambda_t=0$. 
\end{itemize}
Quantitatively, given an upper bound on Eq.(\ref{eq:discounted}) denoted as $\mathrm{Bound}^{\lambda_{1:T}}_T(l_{1:T},u)$, the performance of the player can be characterized by
\begin{equation*}
\sum_{t=1}^T\frac{\prod_{i=t}^{T-1}\lambda_i}{\sum_{t'=1}^T\rpar{\prod_{j=t'}^{T-1}\lambda_j}}l_t(x_t)\leq \underbrace{\sum_{t=1}^T\frac{\prod_{i=t}^{T-1}\lambda_i}{\sum_{t'=1}^T\rpar{\prod_{j=t'}^{T-1}\lambda_j}}l_t(u)}_{\Diamond}+\frac{\mathrm{Bound}^{\lambda_{1:T}}_T(l_{1:T},u)}{\sum_{t'=1}^T\rpar{\prod_{j=t'}^{T-1}\lambda_j}}.
\end{equation*}
Analogous to $T^{-1}\sum_{t=1}^Tl_t(u)$, the term $\Diamond$ measures the comparator losses $l_{1:T}(u)$ on a weighted look-back window, thus could be made small given properly chosen discount factors $\lambda_{1:T}$. In addition, with $\lambda_{1:T}$ fixed, better $\mathrm{Bound}^{\lambda_{1:T}}_T(l_{1:T},u)$ leads to better performance. 

Throughout this paper, we treat $\lambda_{1:T}$ as part of the problem description, and focus on establishing tight upper bounds on Eq.(\ref{eq:discounted}). Choosing $\lambda_{1:T}$ online (i.e., selecting $\lambda_t$ using $g_{1:t+1}$) is an important issue deferred to future works. 

\paragraph{Inductive bias} Exploiting inductive bias for provably better performance is the main practical benefit of adaptive OCO algorithms. Since one would often use the output $x_t$ of an OCO algorithm as the parameter of a machine learning model, we define the inductive bias as a fixed prediction $x^*\in\X$ known at the beginning of the game (possibly obtained from pre-training). Intuitively, predicting $x_t=x^*$ itself would work ``decently well'', and by further modifying it in the game, the goal is to correct its \emph{time-varying imperfection} using the sequentially revealed online data. 

Building on this intuition, it is natural to fix the initialization $x_1=x^*$ in the discounted OCO game. Without loss of generality, the rest of this section will assume $x^*=0$, since a different $x^*$ can be implemented by simply shifting the coordinates. That is, the quantity $\norm{u}$ that will frequently appear in our analysis should be understood as $\norm{u-x^*}$, the distance between the inductive bias $x^*$ and the comparator $u$. 

\subsection{Preliminary}\label{subsection:discounted_preliminary}

We begin by introducing the widespread non-adaptive baseline, constant-LR OGD. To this end, notice that the main difference between the discounted regret Eq.(\ref{eq:discounted}) and the well-studied undiscounted regret ($\lambda_t\equiv 1$) is the \emph{effective time horizon}. Instead of the maximum length $T$, in the $t$-th round Eq.(\ref{eq:discounted}) concerns an exponentially weighted look-back window of length
\begin{equation}\label{eq:effective_H}
H_t\defeq\sum_{i=1}^t\rpar{\prod_{j=i}^{t-1}\lambda^2_j},
\end{equation}
which is roughly $\min[(1-\lambda^2)^{-1},T]$ in the special case of $\lambda_t\equiv\lambda<1$. For later use, we also define the \emph{discounted gradient variance} $V_t$ and the \emph{discounted Lipschitz constant} $G_t$,
\begin{equation}\label{eq:effective_V}
V_t\defeq\sum_{i=1}^t\rpar{\prod_{j=i}^{t-1}\lambda_j^{2}}\norm{g_i}^2;~ G_t\defeq\max_{i\in[1:t]}\rpar{\prod_{j=i}^{t-1}\lambda_j}\norm{g_i}.
\end{equation}

A classical wisdom in online learning is that the learning rates of OGD should be inversely proportional to the square root of the time horizon \citep[Chapter~2]{orabona2023modern}. Combining it with the effective time horizon discussed above, the following result is a folklore. 

\paragraph{Online Gradient Descent} Consider the OGD update rule: after observing the loss gradient $g_t$, we pick a learning rate $\eta_t$, take a gradient step and project the update back to the domain $\X$, i.e., 
\begin{equation*}
x_{t+1}=\Pi_{\X}\rpar{x_t-\eta_t g_t}. 
\end{equation*}

\begin{theorem}[Abridged Theorem~\ref{theorem:ogd} and \ref{theorem:lower}]\label{theorem:abridged}
If the loss functions are all $G$-Lipschitz, the diameter of the domain is at most $D$, and the discount factor $\lambda_t=\lambda\in(0,1)$, then OGD with a constant learning rate $\eta_t=\frac{D}{G}\sqrt{1-\lambda^2}$ guarantees for all $T=\Omega(\frac{1}{1-\lambda})$,
\begin{equation*}
\sup_{l_{1:T},u}\reg^\lambda_T(l_{1:T},u)=O\rpar{ DG\sqrt{H_T}}.
\end{equation*}
Conversely, fix any variance budget $V\in(0,G^2H_T]$, and any comparator $u$ such that $u,-u\in\X$. For any algorithm, there exists a loss sequence such that $V_T$ defined in Eq.(\ref{eq:effective_V}) satisfies $V_T=V$, and
\begin{equation*}
\max\spar{\reg^{\lambda_{1:T}}_T(l_{1:T},u),\reg^{\lambda_{1:T}}_T(l_{1:T},-u)}= \Omega\rpar{\norm{u}\sqrt{V_T}}.
\end{equation*}
\end{theorem}

Picking the domain $\X$ as a norm ball centered at the origin, Theorem~\ref{theorem:abridged} shows that when $\lambda_t\equiv\lambda<1$, the worst case regret bound of constant-LR OGD is \emph{minimax optimal} under the Lipschitzness and bounded-domain assumptions, which forms the theoretical foundation of this common practice. Nonetheless, there is an \emph{instance-dependent} gap that illustrates a natural direction to improve the algorithm: removing the supremum, and directly aiming for
\begin{equation}\label{eq:adaptive_goal}
\reg^{\lambda_{1:T}}_T(l_{1:T},u)=O\rpar{ \norm{u}\sqrt{V_T}}.
\end{equation}
Under the same assumptions, such a bound is never worse than the $O(DG\sqrt{H_T})$ bound of constant-LR OGD, while the associated algorithm can be \emph{agnostic to both $D$ and $G$}. This is the key quantitative strength of adaptivity: without imposing any artificial structural assumption, the algorithm performs as if it knows the ``correct'' assumption from the beginning. 

\subsection{Rescaling trick}\label{subsection:rescaling}

To achieve Eq.(\ref{eq:adaptive_goal}), we propose the following rescaling trick, which is the central component of our analysis. 
\begin{itemize}
\item In the undiscounted setting ($\lambda_t\equiv 1$), Eq.(\ref{eq:adaptive_goal}) has been almost achieved by several recent works \citep{cutkosky2019artificial,mhammedi2020lipschitz,jacobsen2022parameter,zhang2023unconstrained} modulo necessary residual factors. For any such algorithm $\A$, let $\reg_{T,\A}(g_{1:T},u)$ denote its undiscounted regret with respect to linear losses $\inner{g_t}{\cdot}$ and comparator $u$, i.e., Eq.(\ref{eq:discounted}) with $\lambda_t\equiv 1$. 
\item Next, consider the discounted regret Eq.(\ref{eq:discounted}) with general $\lambda_{1:T}$. We take an aforementioned algorithm $\A$ and apply it to a sequence of \emph{surrogate loss gradients} $\hat g_{1:T}$, where
\begin{equation}\label{eq:surrogate_loss}
\hat g_t=\rpar{\prod_{i=1}^{t-1}\lambda_i^{-1}}g_t.
\end{equation}
If $\lambda_1,\ldots,\lambda_T<1$, this amounts to ``upweighting'' recent losses, or equivalently, ``forgetting'' older ones. On the obtained prediction sequence $x_{1:T}$, we evaluate the discounted regret
\begin{align}
\reg^{\lambda_{1:T}}_T(l_{1:T},u)&=\rpar{\prod_{t=1}^{T-1}\lambda_t}\cdot\sum_{t=1}^T\rpar{\prod_{i=1}^{t-1}\lambda^{-1}_i}\spar{l_t(x_t)-l_t(u)}\nonumber\\
&\leq\rpar{\prod_{t=1}^{T-1}\lambda_t}\reg_{T,\A}\rpar{\hat g_{1:T},u}.\label{eq:reduction}
\end{align}
That is, the discounted regret of the weighted algorithm scales with $\reg_{T,\A}\rpar{\hat g_{1:T},u}$, the undiscounted regret of the base algorithm $\A$ on the surrogate loss gradients $\hat g_{1:T}$.
\end{itemize}

Despite its simplicity, such a rescaling trick has a notable subtlety: even if all the actual loss functions $l_t$ are Lipschitz in a \emph{time-uniform} manner ($\exists G, s.t., \max_t\norm{g_t}\leq G$), the surrogate loss functions $\inner{\hat g_t}{\cdot}$ are not. Therefore, the base algorithm $\A$ cannot rely on any \emph{a priori knowledge or estimate} of the time-uniform Lipschitz constant, similar to the \emph{scale-free} property \citep{orabona2018scale} in adaptive online learning.\footnote{A scale-free algorithm generates the same predictions $x_{1:T}$ if all the gradients $g_{1:T}$ are scaled by an arbitrary $c>0$. However, we need a bit more, since an algorithm can be scale-free even if it requires an estimate of the time-uniform Lipschitz constant at the beginning \citep{mhammedi2020lipschitz,jacobsen2022parameter}.} To make this concrete, we first forego the adaptivity to the comparator $u$ and analyze an example based on \emph{gradient adaptive} OGD. The obtained algorithm will be a building block of our main results. 

\paragraph{Gradient adaptive OGD} Proposed for the undiscounted setting, the famous \textsc{AdaGrad} algorithm \cite{duchi2011adaptive} has become a synonym of OGD with \emph{gradient-dependent learning rates}. Using it as the base algorithm $\A$ leads to the following prediction rule similar to \textsc{RMSProp} \citep{tieleman2012lecture}, and its discounted regret bound. 

\begin{restatable}{theorem}{adagrad}\label{theorem:adagrad}
If the diameter of $\X$ is at most $D$, then OGD with learning rate $\eta_t=DV_t^{-1/2}$ guarantees for all $T\in\N_+$ and loss sequence $l_{1:T}$,
\begin{equation*}
\sup_{u}\reg^{\lambda_{1:T}}_T(l_{1:T},u)\leq \frac{3}{2}D\sqrt{V_T}.
\end{equation*}
\end{restatable}

As one would hope for, the bound strictly improves constant-LR OGD while matching the lower bound (Theorem~\ref{theorem:abridged}) on the $\sqrt{V_T}$ dependence. It is tempting to seek an even better learning rate $\eta_t$ that improves the remaining $D$ to $\norm{u}$, but such a direction leads to a dead end, essentially due to the \emph{lack of regularization} (Remark~\ref{remark:ogd_ftrl}). To solve this problem, we will resort to the \emph{Follow the Regularized Leader} (FTRL) framework \citep[Chapter 7]{orabona2023modern} -- another well-known algorithm class in OCO, instead of OGD. Without discounting ($\lambda_t\equiv 1$), its generic procedure\footnote{We introduce FTRL in a slightly unconventional \emph{dual space} formulation; see \citep[Chapter 7.3]{orabona2023modern} for details, and especially, how the prediction function connects to regularization via convex conjugate. Such a dual space formulation is also called the \emph{potential method} in the literature \citep{zhang2022pde}.} is to select a \emph{prediction function} $\phi_t:\R^d\rightarrow\R$ at the beginning of the $t$-th round (possibly depending on past observations), and predict
\begin{equation}\label{eq:ftrl_generic}
x_t=\phi_t\rpar{-\sum_{i=1}^{t-1}g_i}.
\end{equation}

While FTRL and OGD often guarantee the same regret bounds in downstream applications, more refined analyses \citep{fang2022online,jacobsen2022parameter} recently demonstrated that
\begin{quote}
FTRL is stronger in its ability to memorize the initialization, whereas OGD without extra regularization does not. 
\end{quote}
This is particularly important in our discounted setting: FTRL gradually forgets the past online data but not the inductive bias $x^*$, whereas OGD forgets everything altogether. 

\paragraph{FTRL vs regularized OGD} Before switching to our main results, we emphasize that the essential missing piece for OGD is a good regularizer. While regularization is an extra component for OGD, FTRL has it built-in, as its name suggests. To elaborate the connection between FTRL and \emph{regularized} OGD in discounted OCO, we provide the following simple example, with the domain $\X=\R^d$ and the discount factors $\lambda_t\equiv \lambda<1$.
\begin{itemize}
\item First, consider running OGD with learning rate $\eta$. We add an $L_2$ regularizer with weight $\gamma$, which means that the gradient steps are taken on the regularized surrogate losses $f_t(x)\defeq\inner{g_t}{x}+\frac{\gamma}{2}\norm{x}^2$. The resulting prediction rule can be expressed as
\begin{equation*}
x_{t+1}=x_t-\eta g_t-\eta\gamma x_t=-\eta\sum_{i=1}^t\rpar{1-\eta\gamma}^{t-i}g_i.\tag{$x_1=0$}
\end{equation*}
\item Next, consider the following \emph{discounted} FTRL algorithm combining the generic FTRL protocol Eq.(\ref{eq:ftrl_generic}) with the scaled losses Eq.(\ref{eq:surrogate_loss}). Specifically, we pick the prediction function as $\phi_t(x)=c\cdot\lambda^{t-2} x$ with a hyperparameter $c$, which leads to the prediction rule
\begin{equation}\label{eq:ftrl_simple}
x_{t+1}=\phi_{t+1}\rpar{-\sum_{i=1}^{t}\lambda^{1-i}g_i}=- c\sum_{i=1}^{t}\lambda^{t-i}g_i.
\end{equation}
\end{itemize}

Comparing the two cases, one could see that by setting $\eta=c$ and $\gamma=(1-\lambda)\eta^{-1}$, discounted FTRL with a linear prediction function is \emph{equivalent} to $L_2$-regularized OGD. The takeaway is that improving the prediction function in discounted FTRL (which is essentially the direction next) could be thought as refining the $L_2$ regularization, and this can be guided by the existing theory of adaptive OCO. 

\subsection{Simultaneous adaptivity} 

Rigorously, we now use the rescaling trick to achieve the simultaneous adaptivity to both $l_{1:T}$ and $u$, where the algorithm and the analysis will get complicated. As mentioned earlier, there are a number of choices for the base algorithm $\A$. We will adopt the algorithm from \citep{zhang2024improving}, surveyed in Appendix~\ref{subsection:algorithm_prior}, which offers an important benefit (i.e., no explicit $T$-dependence, Remark~\ref{remark:benefit}). Without loss of generality,\footnote{Due to \citep[Theorem~2]{cutkosky2020parameter}, given any unconstrained algorithm that operates on $\X=\R^d$, we can impose any closed and convex constraint without changing its regret bound.} assume the domain $\X=\R^d$. 

\paragraph{Overview of the algorithm} In general, the resulting discounted algorithm employs the \emph{polar-decomposition technique} from \citep{cutkosky2018black}: using polar coordinates, predicting $x_t\in\R^d$ (to ``chase'' the optimal comparator $u$) can be decomposed into two independent tasks, learning the \emph{good direction} $u/\norm{u}$ and the \emph{good magnitude} $\norm{u}$. The direction is learned by the gradient adaptive, \textsc{RMSProp}-like algorithm from Theorem~\ref{theorem:adagrad} (i.e., discounted \textsc{AdaGrad}), while the magnitude is learned by a discounted FTRL-based algorithm that operates on the nonnegative real line $[0,\infty)$ -- more specifically, a discounted variant of the \emph{$\erfi$-potential learner} from \citep[Algorithm~1]{zhang2024improving}. This magnitude learner (presented as Algorithm~\ref{alg:base}) is the central component of this whole procedure, analyzed next in detail. 

\subsubsection{FTRL-based magnitude learner}

\begin{algorithm*}[ht]
\caption{1D FTRL-based magnitude learner on $[0,\infty)$.\label{alg:base}}
\begin{algorithmic}[1]
\REQUIRE Hyperparameter $\eps>0$ (default $\eps=1$). 
\STATE Initialize parameters $v_1=0$, $s_1=0$, $h_1=0$.
\FOR{$t=1,2,\ldots$}
\STATE If $h_t=0$, define the unprojected prediction $\tilde x_t=0$. Otherwise, with $\erfi(x)\defeq\int_0^x\exp(u^2)du$ (which can be queried from \textsc{SciPy} and \textsc{JAX}),
\begin{equation}\label{eq:discounted_prediction_rule}
\tilde x_t=\eps\cdot\erfi\rpar{\frac{s_{t}}{2\sqrt{v_{t}+2h_{t}s_{t}+16h^2_{t}}}}-\frac{\eps h_{t}}{\sqrt{v_t+2h_ts_t+16h_t^2}}\exp\spar{\frac{s_t^2}{4(v_t+2h_ts_t+16h_t^2)}}.
\end{equation}
\STATE Predict $x_t=\Pi_{[0,\infty)}\rpar{\tilde x_t}$, the projection of $\tilde x_t$ to the domain $[0,\infty)$.
\STATE Receive the 1D loss gradient $g_t\in\R$ and the discount factor $\lambda_{t-1}\in(0,\infty)$.
\STATE \label{line:clip} Clip $g_t$ by defining $g_{t,\mathrm{clip}}=\Pi_{[-\lambda_{t-1}h_t,\lambda_{t-1}h_t]}\rpar{g_t}$, and update $h_{t+1}=\max\rpar{\lambda_{t-1}h_t,\abs{g_t}}$.
\STATE \label{line:projection} If $g_{t,\clip}\tilde x_t<g_{t,\clip}x_t$, define a surrogate loss gradient $\tilde g_{t,\clip}=0$. Otherwise, $\tilde g_{t,\clip}=g_{t,\clip}$.
\STATE Update $v_{t+1}=\lambda_{t-1}^2v_{t}+\tilde g^2_{t,\mathrm{clip}}$, $s_{t+1}=\lambda_{t-1}s_{t}-\tilde g_{t,\mathrm{clip}}$.
\ENDFOR
\end{algorithmic}
\end{algorithm*}

The magnitude learner (Algorithm~\ref{alg:base}) has the following intuition. At its center is a special instance of discounted FTRL, Eq.(\ref{eq:discounted_prediction_rule}), which generates the prediction $x_t$ using the \emph{discounted gradient variance} $v_t$, the \emph{discounted gradient sum} $s_t$, and the \emph{discounted Lipschitz constant} $h_t$. This improves the linear prediction function discussed previously, since in comparison, the simple discounted FTRL example Eq.(\ref{eq:ftrl_simple}) can be written as $x_t=c\cdot s_t$. Complementing this core component, two additional ideas are applied to fix certain technical problems: ($i$) the \emph{unconstrained-to-constrained reduction} from \citep{cutkosky2018black,cutkosky2020parameter}, and ($ii$) the \emph{hint-and-clipping technique} from \citep{cutkosky2019artificial}. The readers are referred to Appendix~\ref{subsection:algorithm_prior} for a detailed explanation. 

Notice that understanding the inner workings of the algorithm is not strictly necessary to proceed: we treat the result from \citep{zhang2024improving} as a black box and wrap it using the rescaling trick. In summary, this yields the following theorem proved in Appendix~\ref{subsection:algorithm_detail}.

\begin{restatable}{theorem}{base}\label{theorem:base}
Given any hyperparameter $\eps>0$, Algorithm~\ref{alg:base} guarantees for all time horizon $T\in\N_+$, loss sequence $l_{1:T}$, comparator $u\in[0,\infty)$ and stability window length $\tau\in[1:T]$,
\begin{equation*}
\reg^{\lambda_{1:T}}_T(l_{1:T},u)\leq \eps\sqrt{V_T+2G_TS+16G_T^2}+u\rpar{S+G_T}+\rpar{\max_{t\in[T-\tau+1:T]} x_t}G_T+\rpar{\prod_{t=T-\tau}^{T-1}\lambda_t}\rpar{\max_{t\in[1 :T-\tau]} x_t}G_{T-\tau},
\end{equation*}
where
\begin{equation*}
S= 8G_T\rpar{1+\sqrt{\log(2 u\eps^{-1}+1)}}^2+2\sqrt{V_T+16G_T^2}\rpar{1+\sqrt{\log(2u\eps^{-1}+1)}}.
\end{equation*}
\end{restatable}

Let us take a few steps to interpret this result. In particular, we justify the appropriate asymptotic regime to consider ($V_T\gg G_T^2$ and $u\gg \eps$), such that our main result on $\R^d$ (Theorem~\ref{theorem:main}) can use the big-Oh notation to improve clarity. 
\begin{itemize}
\item First, one would typically expect $V_T\gg G_T^2$, since when $\lambda_t=\lambda\in(0,1)$ and $\abs{g_t}=G$ for all $t$, we have $G_T=G$ and $V_T=H_TG^2\approx (1-\lambda^2)^{-1}G^2$. As long as the discount factor $\lambda_t$ is close enough to $1$, the condition $V_T\gg G_T^2$ is likely to hold for general/practical gradient sequences as well. 
\item Second, the hyperparameter $\eps$ serves as a prior guess of the comparator $u$. If the guess is correct ($\eps=u$), then by assuming $V_T\gg G_T^2$ and $\max_{t\in[1:T]}x_t=O(u)$ (roughly speaking, the predictions are \emph{stable}), the regret bound becomes
\begin{equation}\label{eq:oracle}
\reg^{\lambda_{1:T}}_T(l_{1:T},u)=O\rpar{u\sqrt{V_T}},
\end{equation}
exactly matching the lower bound in Theorem~\ref{theorem:abridged}. Realistically such an ``oracle tuning'' is illegal, since $\eps$ is selected at the beginning of the game, while reasonable comparators $u$ are hidden before all the loss functions are revealed. 

This is where our algorithm shines: as long as $\eps$ is moderately small, i.e., $O(u)$, we would 
have the $O(uS)$ term dominating the regret bound, which only suffers a multiplicative \emph{logarithmic penalty} relative to the impossible oracle-optimal rate Eq.(\ref{eq:oracle}). In comparison, it is well-known that the regret bound of constant-LR OGD with learning rate $\eta$ depends \emph{polynomially} on $\eta$ and $\eta^{-1}$ (cf., Appendix~\ref{subsection:preliminary_proofs}), which means our algorithm is provably more robust to suboptimal hyperparameter tuning. 
\item Third, we explain the use of $\tau$. In nonstationary environments, the range of $x_t$ can vary significantly over time, therefore a time-uniform characterization of its stability ($\max_{t\in[1:T]}x_t$, Remark~\ref{remark:stability}) could be overly conservative. We use a \emph{stability window} of arbitrary length $\tau$ to divide the time horizon into two parts: the earlier part is forgotten rapidly (due to the $\prod_{t=T-\tau}^{T-1}\lambda_t$ multiplier), so only the later part really matters. This is further discussed in the following example. 
\end{itemize}

\begin{example}
Suppose again that $\lambda_t=\lambda\in(0,1)$. If $\lambda\approx 1$, the ``forgetting'' multiplier can be approximated by
\begin{equation*}
\prod_{t=T-\tau}^{T-1}\lambda_t=\lambda^{\tau}=\lambda^{\frac{1}{1-\lambda}\cdot(1-\lambda)\tau}\approx e^{(\lambda-1)\tau}.\tag{Lemma~\ref{lemma:e}}
\end{equation*}
Consider $\lambda=0.99$ for example: $\tau=700$ ensures $\lambda^{\tau}\approx e^{-7}<10^{-3}$. That is, if the past range of $x_t$ is negligible after a $10^{-3}$-attenuation, then our regret bound only depends on the ``localized stability'' $\max_{t\in[T-699,T]}x_t$ evaluated in the recent 700 rounds, which is a small fraction in (let us say) $T\approx\textrm{million}$. More intuitively, it means ``past mistakes do not matter''.
\end{example}

\subsubsection{The combined algorithm}

Given the 1D magnitude learner and its guarantee, the extension to $\R^d$ follows from a standard polar-decomposition technique \cite{cutkosky2018black}. We defer the pseudocode to Appendix~\ref{subsection:algorithm_detail}.

\begin{theorem}[Main result]\label{theorem:main}
Given any hyperparameter $\eps>0$, Algorithm~\ref{alg:meta} in Appendix~\ref{subsection:algorithm_detail} guarantees for all $T\in\N_+$, loss gradients $g_{1:T}$ and comparator $u\in\R^d$,
\begin{equation*}
\reg^{\lambda_{1:T}}_T(l_{1:T},u)\leq O\rpar{\norm{u}\sqrt{V_T\log(\norm{u}\eps^{-1})}\vee\norm{u}G_T\log(\norm{u}\eps^{-1})}+\rpar{\max_{t\in[1 :T]} x_t}G_T,
\end{equation*}
where $O(\cdot)$ is in the regime of large $V_T$ $(V_T\gg G_T)$ and large $\norm{u}$ $(\norm{u}\gg \eps)$. Furthermore, for $u=0$, 
\begin{equation*}
\reg^{\lambda_{1:T}}_T(l_{1:T},0)\leq O\rpar{\eps\sqrt{V_T}}+\rpar{\max_{t\in[1 :T]} x_t}G_T.
\end{equation*}
\end{theorem}

Similar to Theorem~\ref{theorem:base}, the iterate stability term can be split into two parts using an arbitrary $\tau$. The key message is that if the iterates are indeed stable ($\max_{t\in[1 :T]} x_t=O(\norm{u})$) and we suppress all the logarithmic factors with $\tilde O(\cdot)$, then
\begin{equation*}
\reg^{\lambda_{1:T}}_T(l_{1:T},u)\leq \tilde O\rpar{\norm{u}\sqrt{V_T}}.
\end{equation*}
It matches the lower bound and improves all the aforementioned algorithms. Although there are several nuances in this statement, they are in general necessary even in the undiscounted special case \citep[Chapter~9]{orabona2023modern}. 

\paragraph{Practical strengths} Finally, we summarize a range of practical strengths. As shown earlier, the proposed algorithm is robust to hyperparameter tuning. Observations and mistakes from the distant past (which are possibly misleading for the future) are appropriately forgotten, such that the algorithm runs ``consistently'' over its lifetime. Compared to the typical aggregation framework in nonstationary online learning \citep{daniely2015strongly,zhang2018adaptive}, our algorithm runs faster and never explicitly restarts (although we require given discount factors to ``softly'' restart). In addition, compared to constant-LR OGD that also tries to minimize the discounted regret, our algorithm makes a better use of the inductive bias $x^*$ -- in the general case of $x^*\neq 0$, our discounted regret bound scales with $\norm{u-x^*}$. This exemplifies the crucial role of regularization, building on the adaptive OCO theory. 

\section{Online conformal prediction}\label{section:conformal}

Next, we switch topic and consider a problem that complements the above discounted OCO theory. Conformal prediction \cite{vovk2005algorithmic} is a framework that quantifies the uncertainty of black box ML models. We study its online version called \emph{Online Conformal Prediction} (OCP) \cite{gibbs2021adaptive,bastani2022practical,gibbs2022conformal,zaffran2022adaptive,bhatnagar2023improved}, where no statistical assumptions (e.g., exchangeability) are imposed at all. Our setting follows the nicely written \cite[Section~2]{bhatnagar2023improved}, and the readers are referred to \cite{angelopoulos2023conformal,roth2022uncertain} for additional background. 

\subsection{Preliminary}

OCP is closely tied to the problem of (adversarial) quantile regression. Similar to OCO, it is a two-person repeated game. Let a constant $\alpha\in(0,1)$ be the \emph{targeted miscoverage rate} fixed before the game starts. At the beginning of the $t$-th round, we receive a set-valued function $\calC_t:\R_{\geq 0}\rightarrow 2^\Y$ mapping any \emph{radius parameter} $r\in[0,\infty)$ to a subset $\calC_t(r)$ of the \emph{label space} $\Y$. The $\calC_t$ function is \emph{nested}: for any $r'> r$, we have $\calC_t(r)\subset\calC_t(r')\subset \Y$. Then,
\begin{enumerate}
\item We pick a radius parameter $r_t\in[0,\infty)$ and output the prediction set $\calC_t(r_t)$.
\item The environment reveals the \emph{optimal radius} $r^*_t\in[0,\infty)$. Intuitively, our prediction set $\calC_t(r_t)$ is ``large enough'' only if $r_t>r^*_t$. 
\item Our performance is evaluated by the \emph{pinball loss} $l^{(\alpha,*)}_t(r_t)$, where for all $r\in[0,\infty)$,
\begin{equation}
l^{(\alpha,*)}_t(r)\defeq\begin{cases}
\alpha(r-r_t^*),& r> r^*_t,\\
(\alpha-1)(r-r^*_t),& \textrm{else}.
\end{cases}\label{eq:pinball}
\end{equation}
\end{enumerate}

Here is a simple 1D forecasting example. Suppose there is a base ML model that in each round makes a prediction $\hat x_t\in\R$ of the true time series $x^*_t\in\R$. On top of that, we ``wrap'' such a point prediction $x_t$ by a confidence set prediction $\calC_t(r_t)=(\hat x_t-r_t,\hat x_t+r_t)$. Ideally we want $\calC_t(r_t)$ to cover the true series $x^*_t$, and this can be checked after $x^*_t$ is revealed. That is, by defining the optimal radius $r^*_t=\abs{x^*_t-\hat x_t}$, we claim success if $r_t> r^*_t$.

Quite naturally, since $r^*_t$ is arbitrary, it is impossible to \emph{ensure} coverage unless $r_t$ is meaninglessly large. A reasonable objective is then asking our \emph{empirical (marginal) coverage rate} to be approximately $1-\alpha$, which amounts to showing
\begin{equation}\label{eq:coverage}
\abs{\frac{1}{T}\sum_{t=1}^T\bm{1}\spar{r_t\leq r^*_t}-\alpha}=o(1),
\end{equation}
and this is beautifully \emph{equivalent} to characterizing the cumulative subgradients of the pinball loss,\footnote{At the singular point $r^*_t$, define the subgradient $\partial l_t^{(\alpha,*)}(r^*_t)=\alpha-1$.} $\abs{\sum_{t=1}^T\partial l^{(\alpha,*)}_t(r_t)}=o(T)$. One catch is that there are trivial predictors\footnote{Alternating between $r_t=\infty$ and $r_t=0$ in a pattern independent of data \citep{bastani2022practical}.} satisfying Eq.(\ref{eq:coverage}) \cite{bastani2022practical}, so one needs an extra measure to rule them out. Such a ``secondary objective'' can be the regret bound on the pinball loss,\footnote{Another choice \cite{bastani2022practical} is bounding the \emph{conditional coverage} using \emph{calibration}.}
\begin{equation}\label{eq:ocp_reg_undiscounted}
\sum_{t=1}^Tl^{(\alpha,*)}_t(r_t)-\sum_{t=1}^Tl^{(\alpha,*)}_t(u)=o(T),
\end{equation}
which motivates using OCO algorithms such as OGD to select $r_t$. 

How does nonstationarity enter the picture? Since the proposal of OCP in \cite{gibbs2021adaptive}, the main emphasis is on problems with \emph{distribution shifts}, which traditional conformal prediction methods based on exchangeability and data splitting have trouble dealing with. For example, the popular \textsc{ACI} algorithm \cite{gibbs2021adaptive} essentially uses constant-LR OGD -- as we have shown, this is inconsistent with minimizing the standard regret Eq.(\ref{eq:ocp_reg_undiscounted}), and the ``right'' OCO performance metric that justifies it could be a nonstationary one (e.g., the discounted regret). In a similar spirit, \cite{gibbs2022conformal,bhatnagar2023improved} applied \emph{dynamic} and \emph{strongly adaptive} OCO algorithms, effectively analyzing the subinterval variants of Eq.(\ref{eq:coverage}) and Eq.(\ref{eq:ocp_reg_undiscounted}). 

Another key ingredient of OCP is assuming the optimal radius $\max_tr^*_t\leq D$ for some $D>0$, which is often reasonable in practice and important for the coverage guarantee. As opposed to prior works \cite{bastani2022practical,bhatnagar2023improved} that require \emph{knowing} $D$ at the beginning to initialize properly, we seek an adaptive algorithm agnostic to this oracle knowledge. 

\paragraph{Our goal} Overall, we aim to show that without knowing $D$, applying Algorithm~\ref{alg:base} leads to \emph{discounted adaptive} versions of the marginal coverage bound Eq.(\ref{eq:coverage}) and the regret bound Eq.(\ref{eq:ocp_reg_undiscounted}). This offers advantages over the \textsc{ACI}-like approach that tackles similar sliding window objectives using constant-LR OGD. Along the way, we demonstrate how the structure of OCP allows controlling the iterate stability of Algorithm~\ref{alg:base} (or generally, FTRL algorithms), which then makes the proof of coverage fairly easy. 

\subsection{Main result}

\paragraph{Beyond pinball loss} From now on, define $\A_{CP}$ as the OCP algorithm that uses Algorithm~\ref{alg:base} to select $r_t$ (see Appendix~\ref{subsection:ocp_subroutine} for pseudocode). We make a major generalization: instead of using subgradients of the pinball loss Eq.(\ref{eq:pinball}) to update Algorithm~\ref{alg:base}, we use subgradients $g^*_t\in \partial f^*_t(r_t)$, where $f^*_t(r)$ is any convex function minimized at $r^*_t$, and right at $r_t=r^*_t$ we have $g^*_t\leq 0$ without loss of generality. This includes the pinball loss $l^{(\alpha,*)}_t(r)$ as a special case. Notably, $f^*_t(r)$ does not need to be globally Lipschitz,\footnote{For example, one might use a ``skewed quadratic function'' $f^*_t(r)=\half\rpar{\alpha-\bm{1}[r\leq r^*_t]}(r-r^*_t)^2$ to penalize the under/over-coverage margin.} which unleashes the full power of our base algorithm. Put together, $\A_{CP}$ takes a confidence hyperparameter $\eps$ and a sequence of discount factors $\lambda_{1:T}$ determined by our objectives. We assume $\max_t r^*_t\leq D$, but $D$ is unknown by $\A_{CP}$. 

\paragraph{Strength of FTRL} The key to our result is Lemma~\ref{lemma:connecting_abridged} connecting the prediction $r_{t+1}$ to the \emph{discounted coverage metric}
\begin{equation}\label{eq:discounted_g_sum}
S^*_t\defeq-\sum_{i=1}^t\rpar{\prod_{j=i}^{t-1}\lambda_j}g^*_i.
\end{equation}
If $\lambda_t=1$ for all $t$ and we use the pinball loss to define $g^*_{1:T}$, then $\abs{S^*_T}/T$ recovers Eq.(\ref{eq:coverage}). The point is that if $\lambda_t=\lambda<1$, then just like the intuition throughout this paper, Eq.(\ref{eq:discounted_g_sum}) is essentially a sliding window coverage metric. Associated algorithm would gradually forget the past, which intuitively counters the distribution shifts. 

\begin{lemma}[Abridged Lemma~\ref{lemma:connecting}]\label{lemma:connecting_abridged}
$\A_{CP}$ guarantees for all $t\in\N_+$, 
\begin{equation*}
\abs{S^*_t}\leq O\rpar{\sqrt{V^*_t\log(r_{t+1}\eps^{-1})}\vee G^*_t\log(r_{t+1}\eps^{-1})},
\end{equation*}
where $V^*_t$ and $G^*_t$ are defined on the OCP loss gradients using Eq.(\ref{eq:effective_V}), and $O(\cdot)$ is in the regime of $r_{t+1}\gg \eps$.
\end{lemma}

The proof of this lemma is a bit involved, but the high level idea is very simple: if we use a FTRL algorithm (rather than OGD) as the OCO subroutine for OCP, then the radius prediction is roughly $r_{t+1}\approx\psi(S^*_t/\sqrt{V^*_t})$ for some function $\psi$ (if the algorithm is not ``adaptive enough'' then the denominator is $\sqrt{H_t}$ instead), which means $S^*_t\approx\sqrt{V^*_t}\psi^{-1}(r_{t+1})=O(\sqrt{H_t})$. Dividing both sides by $H_t$ yields the desirable coverage guarantee. In comparison, the parallel analysis using OGD can be much more complicated (e.g., the grouping argument in \cite{bhatnagar2023improved}) due to the absence of $S^*_t$ in the explicit update rule. Such an analytical strength of FTRL seems to be overlooked in the OCP literature. 

We also note that although the above lemma depends only logarithmically on $r_{t+1}$, without any problem structure the latter could still be large (exponential in $t$), which invalidates this approach. The remaining step is showing that if the underlying optimal radius $r^*_t$ is time-uniformly bounded by $D$, then even without knowing $D$, we could replace $r_{t+1}$ in the above lemma by $O(D)$. Intuitively it should make sense; materializing it carefully gives us the final result. 

\begin{restatable}{theorem}{ocp}\label{theorem:ocp_main}
Without knowing $D$, $\A_{CP}$ guarantees that for all $T\in\N_+$, we have the discounted coverage bound
\begin{equation*}
\abs{S^*_T}\leq O\rpar{\sqrt{V^*_T\log(D\eps^{-1})}\vee G^*_T\log(D\eps^{-1})},
\end{equation*}
and the discounted regret bound from Theorem~\ref{theorem:base}.
\end{restatable}

To interpret this result, let us focus on the coverage bound, since the discount regret bound has been discussed extensively in Section~\ref{section:main}. First, consider the pinball loss and the undiscounted setting $\lambda_t=1$, where our bound can be directly compared to prior works. For OGD, \cite[Proposition~1]{gibbs2021adaptive} shows that the learning rate $\eta=D/\sqrt{T}$ (as suggested by regret minimization) achieves $\abs{S^*_T}=O(\sqrt{T})$, while \cite[Theorem~2]{bhatnagar2023improved} shows that $\eta_t=D/\sqrt{V_t}$ (i.e., \textsc{AdaGrad}) achieves $\abs{S^*_T}=O(\alpha^{-2}T^{3/4}\log T)$. Although the latter is empirically strong, the theory is a bit unsatisfying as one would expect the gradient adaptive approach to be a ``pure upgrade'' (plus, the bound blows up as $\alpha\rightarrow 0$). Our Theorem~\ref{theorem:ocp_main} is in some sense the ``right fix'', as essentially, $\abs{S^*_T}=\tilde O(\sqrt{V^*_T})\leq \tilde O(\sqrt{T})$. This is primarily due to the strength of FTRL over OGD, which we hope to demonstrate. Besides, our algorithm also improves the regret bound of these baselines while being agnostic to $D$. 

All these algorithms can be extended to the sliding window setting (e.g., the algorithm from \cite{gibbs2021adaptive} becomes constant-LR OGD, which is the version actually applied in practice), and the above comparison still roughly holds. There is just one catch: OGD algorithms make coverage guarantees on ``exact sliding windows'' $[T-H+1:T]$ of length $H$, whereas our algorithm bounds coverage on a slightly different ``exponential window'' of effective length $H_T$, Eq.(\ref{eq:discounted_g_sum}). Nonetheless, all these algorithms also guarantee certain discounted regret bounds, which are still comparable like in Section~\ref{section:main}. 

\paragraph{Adaptivity in OCP} Next, we discuss the benefits of adaptivity more concretely in OCP. Suppose again that $\lambda_t=1$ and we use the pinball loss. Then, instead of the non-adaptive bound $\abs{S^*_T}=O(\sqrt{T})$, we have
\begin{equation*}
\abs{S^*_T}=\tilde O\rpar{\sqrt{V^*_T}}=\tilde O\rpar{\sqrt{\alpha^2\sum_{t=1}^T\bm{1}[r_t>r^*_t]+(\alpha-1)^2\sum_{t=1}^T\bm{1}[r_t\leq r^*_t]}}.
\end{equation*}
Since asymptotically the miscoverage rate is $\alpha$, we have $\sum_{t=1}^T\bm{1}[r_t\leq r^*_t]\approx \alpha T$, which means the bound is roughly $\tilde O\rpar{\sqrt{\alpha(1-\alpha)T}}$. That is, the gradient adaptivity in OCO translates to the \emph{target rate adaptivity} in OCP. 

In terms of the sample complexity, the OGD baseline requires at least $\eps^{-2}$ rounds to guarantee an empirical marginal coverage rate within $[\alpha-\eps,\alpha+\eps]$, while our algorithm requires at least $\alpha\eps^2$ rounds. Very concretely, in a typical setting of $\alpha=0.1$ (i.e., we want $90\%$ confidence sets), our algorithm only requires $\frac{1}{10}$ as many samples compared to the OGD baseline. Besides the coverage bound, such an $\alpha$-dependent saving applies to the regret bound as well. 

\paragraph{Initialization in OCP} Finally, we comment on the role of initialization in OCP, mirroring the discussion at the end of Section~\ref{section:main}. There, the idea is that FTRL always remembers the initialization, therefore we might inject prior knowledge by choosing it properly. A possibly interesting observation is that initializing at $r_1=0$ is a sensible choice in OCP. It always ``drags'' the radius prediction towards zero, such that intuitively, ``within all the candidate prediction sets satisfying the coverage and regret bounds, we pick the smallest one''.

\section{Experiment}\label{section:experiment}

We now demonstrate the practicality of our algorithm in OCP experiments. Our setup closely builds on \citep{bhatnagar2023improved}. Except our own algorithms, we adopt the implementation of the baselines and the evaluation procedure from there. Details are deferred to Appendix~\ref{section:more_experiment}, and code is available at \url{https://github.com/ComputationalRobotics/discounted-adaptive}.

\paragraph{Setup} We consider image classification in a sequential setting, where each image is subject to a corruption of time-varying strength. Given a base machine learning model that generates the $\calC_t$ function (by scoring all the possible labels), the goal of OCP is to select the radius parameter $r_t$, which yields a set of predicted labels. Ideally, we want such a set to contain the true label, while being as small as possible. The targeted miscoverage rate $\alpha$ is selected as $0.1$. 

We test three versions of our algorithm: \textsc{MagL-D} is our Algorithm~\ref{alg:base} with $\eps=1$ and $\lambda_t=0.999$; \textsc{MagL} is its undiscounted version ($\lambda_t=1$); and \textsc{MagDis} is a much simplified variant of Algorithm~\ref{alg:base} that basically sets $h_t=0$. We are aware that a possible complaint towards our approach is that the algorithm is too complicated, but as we will show, this simplified version presented as Algorithm~\ref{alg:zero_ht} also demonstrates strong empirical performance despite losing the performance guarantee. 

The baselines we test are summarized in Table~\ref{tab:methods_performance}, following the implementation of \cite{bhatnagar2023improved}. In particular, \textsc{Sf-Ogd} \cite{bhatnagar2023improved} is equivalent to \textsc{AdaGrad} with \emph{oracle tuning}: by definition it requires knowing the maximum possible radius $D$ to set the learning rate, and in practice, $D$ is estimated from an offline dataset (which is a form of oracle tuning from the theoretical perspective). To test the effect of such tuning, we create another baseline called ``Simple OGD'', which is simply \textsc{Sf-Ogd} with its estimate of $D$ set to 1. We emphasize that despite its name, Simple OGD is still a gradient adaptive algorithm. 

Four metrics are evaluated, and we define them formally in Appendix~\ref{section:more_experiment}. First, the \emph{average coverage} measures the empirical coverage rate over the entire time horizon. Similarly, the \emph{average width} refers to the average size of the prediction set, also over the entire time horizon. Ideally we want the average coverage to be close to $1-\alpha=0.9$, and if that is satisfied, lower average width is better. Different from these two, the \emph{local coverage error} ($\mathrm{LCE}_{100}$) measures the deviation of the \emph{local} empirical coverage rate (over the ``worst'' sliding time window of length $100$) from the target rate $0.9$ -- this is arguably the most important metric (due to the distribution shifts), and lower is better. Finally, we also test the runtime of all the algorithms, normalized by that of Simple OGD. 

\begin{table}[t]
\centering
\begin{tabular}{|l|c|c|c|c|c|c|}
\hline
\textit{Method} & \textit{Avg. Coverage} & \textit{Avg. Width} & $\mathrm{LCE}_{100}$ &
\textit{Runtime}  \\
\hline
    Simple OGD        & $\mathbf{0.899}$ &           $125.0$   & $0.11$ &    $\mathbf{1.00 \pm 0.03}$ \\ 
    \textsc{Sf-Ogd}    & $\mathbf{0.899}$ &           $124.5$   & $0.09$ & \underline{$1.05 \pm 0.03$} \\ 
    \textsc{Saocp}    & $0.883$ &            \underline{$119.2$}  & $0.10$ & $11.07 \pm 0.19$ \\
    SplitConformal   & $0.843$ &            $129.5$  & $0.47$ & $1.57 \pm 0.04$ \\ 
    NExConformal     & $0.892$ &            $123.0$  & $0.14$ &            $2.22 \pm 0.02$ \\ 
    \textsc{Faci}    & $0.889$ &            $123.4$  & $0.12$ &            $2.98 \pm 0.07$ \\  
    \textsc{Faci-S}  & $0.892$ &            $124.7$  & $0.11$ &            $2.17 \pm 0.07$ \\ 
    \textcolor{blue}{Alg. \ref{alg:base}: \textsc{MagL-D}} & 
    \textcolor{blue}{0.884} & 
    \textcolor{blue}{$\mathbf{118.5}$} & 
    \textcolor{blue}{\underline{0.08}} & 
    \textcolor{blue}{$1.06 \pm 0.03$} \\ 
    \textcolor{blue}{Alg. \ref{alg:base_undiscounted}: \textsc{MagL}} & 
    \textcolor{blue}{\underline{$0.894$}} & 
    \textcolor{blue}{122.1} & 
    \textcolor{blue}{0.09} & 
    \textcolor{blue}{$1.05 \pm 0.02$} \\ 
    \textcolor{blue}{Alg. \ref{alg:zero_ht}: \textsc{MagDis}} &
    \textcolor{blue}{0.888} & 
    \textcolor{blue}{122.1} & 
    \textcolor{blue}{$\mathbf{0.07}$} & 
    \textcolor{blue}{$1.02 \pm 0.04$}\\ 
\hline
\end{tabular}

\caption{Performance of different methods. Baselines are colored in black, while our algorithms are colored in \textcolor{blue}{blue}. The best performer in each metric is \textbf{bolded} and the second-best is \underline{underlined}. The runtime, plus/minus its standard deviation, is normalized by the mean of Simple OGD's runtime; averages are take over ten trials per algorithm.}
\label{tab:methods_performance}
\end{table}

\paragraph{Result} The results are summarized in Table~\ref{tab:methods_performance}. Among all the algorithms tested, our algorithms achieve the lowest local coverage error ($\mathrm{LCE}_{100}$). In terms of metrics on the entire time horizon, our algorithms also demonstrate competitive performance compared to the baselines: although the average coverage is worse than that of Simple OGD and \textsc{Sf-Ogd},\footnote{This is reasonable since Simple OGD and \textsc{Sf-Ogd} are both undiscounted static regret minimization algorithms, which should perform well on ``global'' metrics.} the average width is lower.\footnote{In general, we find that our algorithms tend to favor low width at the price of slightly worse coverage. As discussed at the end of Section~\ref{section:conformal}, this might be explained by the initialization at zero.} In addition, our algorithms run almost as fast as Simple OGD and \textsc{Sf-Ogd}, and importantly, they are significantly faster than \textsc{Saocp} \cite{bhatnagar2023improved} which is an aggregation algorithm that minimizes the strongly adaptive regret. 

By comparing the three versions of our algorithm, as one would expect, the discounted version (\textsc{MagL-D}) improves the undiscounted version (\textsc{MagL}). Remarkably, the much simplified \textsc{MagDis} achieves competitive performance despite the lack of performance guarantees. 

\section{Conclusion}

Motivated by the crucial role of inductive bias in nonstationary online learning, this work revisits discounted OCO using recently developed techniques in adaptive algorithms. In particular, we propose a discounted ``simultaneously'' adaptive algorithm (with respect to both the loss sequence and the comparator), with demonstrated benefits in online conformal prediction. Along the way, we propose ($i$) a simple rescaling trick to minimize the discounted regret; and ($ii$) a FTRL-based analytical strategy to guarantee coverage rate in online conformal prediction. More broadly, our results suggest that the adaptive OCO theory may help improve the conventional regularization mechanisms in practical lifelong learning. 

Moving forward, we hope this work can help revive the community's interest in the discounted regret. For nonstationary online learning, it is a simpler metric to study than the alternatives, while offering certain practical advantages (Appendix~\ref{section:more_related}). We leave the online selection of the discount factors as an important open question. Another interesting direction is to test adaptive-OCO-based regularizers in realistic scale, non-convex lifelong learning problems, for which our follow-up work \citep{muppidi2024pick} presents promising results. Furthermore, as suggested by \cite{cutkosky2023optimal,ahn2024understanding}, there could be substantial theoretical connections between discounting and deep learning optimization. 

\section*{Acknowledgement}

This project is partially funded by Harvard University Dean’s Competitive Fund for Promising Scholarship.

\bibliography{Discounted}

\newpage
\section*{Appendix}
\appendix

\paragraph{Organization} Appendix~\ref{section:more_related} surveys the related topic of dynamic and strongly adaptive regret. Appendix~\ref{section:main_detail} and \ref{section:conformal_detail} contain the details of discounted online learning and the OCP application, respectively. Details of the experiments are presented in Appendix~\ref{section:more_experiment}.

\section{Dynamic and strongly adaptive regret}\label{section:more_related}

Recent works on nonstationary online learning predominantly focused on aggregation algorithms that minimize either the dynamic regret or the strongly adaptive regret. They are both based on the undiscounted OCO game (i.e., the setting of Section~\ref{section:main} with $\lambda_t\equiv 1$), with different ways to quantitatively measure the nonstationarity level of the environment. 

\paragraph{Dynamic regret} Generalizing the \emph{static} regret Eq.(\ref{eq:discounted}), the \emph{Zinkevich-style dynamic regret} \citep{zinkevich2003online,zhang2018adaptive,zhang2018dynamic,zhao2020dynamic,baby2021optimal,baby2022optimal,jacobsen2022parameter,lu2023computational,zhang2023unconstrained} allows the comparator $u\in\X$ to be time-varying. The goal is to upper-bound
\begin{equation*}
\reg_T(l_{1:T},u_{1:T})\defeq\sum_{t=1}^Tl_t(x_t)-\sum_{t=1}^Tl_t(u_t),
\end{equation*}
for all losses $l_{1:T}$ and comparator sequences $u_{1:T}$.

\paragraph{Strongly adaptive regret} The \emph{strongly adaptive regret} \cite{hazan2009efficient,daniely2015strongly,adamskiy2016closer,jun2017improved,cutkosky2020parameter,lu2023computational} generalizes Eq.(\ref{eq:discounted}) to \emph{subintervals} of the time horizon. The goal is to upper-bound
\begin{equation*}
\reg_\I(l_{\I},u)\defeq\sum_{t\in\I}l_t(x_t)-\sum_{t\in\I}l_t(u),
\end{equation*}
for all losses $l_{1:T}$, comparator $u$, and time interval $\I\subset[1:T]$. 

\paragraph{Aggregation algorithm} Due to the intricate connection between these performance metrics \cite{zhang2018dynamic,cutkosky2020parameter,baby2021optimal}, existing algorithms all share the idea of bi-level aggregation:
\begin{enumerate}
\item The low level maintains a class of base online learning algorithms in parallel, each targeting a different nonstationarity level of the environment that is possibly correct. 
\item The high level aggregates these base algorithms, in order to adapt to the true nonstationarity level unknown beforehand.
\end{enumerate}

Theoretically this works well, but there is a less-noticed limitation: bounding either performance metric above requires targeting a very wide range of possible nonstationarity levels (of the environment). It means that inevitably, the per-round computation increases over time, and roughly speaking, there is not enough forgetting on the obsolete data. Discounting is a possible way to bypass these issues in practice (although theoretically it targets the discounted regret which is incomparable). 

\section{Detail of Section~\ref{section:main}}\label{section:main_detail}

Appendix~\ref{subsection:preliminary_proofs} contains omitted proofs from Section~\ref{section:main}, excluding the analysis of our main algorithm. Appendix~\ref{subsection:algorithm_prior} introduces the undiscounted algorithm from \cite{zhang2024improving} and its guarantees; these are applied to our reduction from Section~\ref{subsection:rescaling}. Appendix~\ref{subsection:algorithm_detail} proves our main result, i.e., discounted regret bounds that adapt simultaneously to $l_{1:T}$ and $u$. 

\subsection{Preliminary proofs}\label{subsection:preliminary_proofs}

We first prove an auxiliary lemma. 

\begin{lemma}\label{lemma:e}
For all $\lambda\in(0,1)$, $\lambda^{\frac{1}{1-\lambda}}\leq e^{-1}$. Moreover, $\lim_{\lambda\rightarrow 1^-}\lambda^{\frac{1}{1-\lambda}}=e^{-1}$.
\end{lemma}

\begin{proof}[Proof of Lemma~\ref{lemma:e}]
For the first part of the lemma, taking $\log$ on both sides, it suffices to show
\begin{equation*}
\frac{1}{1-\lambda}\log\lambda\leq -1.
\end{equation*}
This holds due to $\log\lambda\leq \lambda-1$. The second part is due to $\lim_{\lambda\rightarrow 1^-}(1-\lambda)^{-1}\log\lambda=-1$.
\end{proof}

The following theorem characterizes non-adaptive OGD. 

\begin{theorem}\label{theorem:ogd}
If the loss functions are all $G$-Lipschitz and the diameter of the domain $\mathrm{diam}(\X)$ is at most $D$, then OGD with learning rate $\eta_t=DG^{-1}H^{-1/2}_t$ guarantees for all $T\in\N_+$,
\begin{equation*}
\sup_{l_{1:T},u}\reg^{\lambda_{1:T}}_T(l_{1:T},u)\leq \frac{3}{2}DG\sqrt{H_T}.
\end{equation*}
Furthermore, if the discount factor $\lambda_t=\lambda$ for some $\lambda\in(0,1)$, then OGD with a time-invariant learning rate $\eta_t=DG^{-1}\sqrt{1-\lambda^2}$ guarantees for all $T
\geq \half(1-\lambda)^{-1}$,
\begin{equation*}
\sup_{l_{1:T},u}\reg^\lambda_T(l_{1:T},u)\leq \frac{3}{2}\frac{DG}{\sqrt{1-\lambda^2}}\leq \frac{3}{2\sqrt{1-e^{-1}}}DG\sqrt{H_T}.
\end{equation*}
\end{theorem}

\begin{proof}[Proof of Theorem~\ref{theorem:ogd}]
We start with the first part of the theorem. The standard analysis of gradient descent centers around the following lemma \cite[Lemma~2.12]{orabona2023modern}, 
\begin{equation*}
\inner{g_t}{x_t-u}\leq \frac{1}{2\eta_t}\norm{x_t-u}^2-\frac{1}{2\eta_t}\norm{x_{t+1}-u}^2+\frac{\eta_t}{2}\norm{g_t}^2.
\end{equation*}
Applying convexity and taking a telescopic sum, 
\begin{align}
&\reg^{\lambda_{1:T}}_T(l_{1:T},u)\nonumber\\
=~&\sum_{t=1}^T\rpar{\prod_{i=t}^{T-1}\lambda_i}\spar{l_t(x_t)-l_t(u)}\nonumber\\
\leq~&\sum_{t=1}^T\rpar{\prod_{i=t}^{T-1}\lambda_i}\inner{g_t}{x_t-u}\nonumber\\
\leq~& \sum_{t=1}^T\rpar{\prod_{i=t}^{T-1}\lambda_i}\rpar{\frac{1}{2\eta_t}\norm{x_t-u}^2-\frac{1}{2\eta_t}\norm{x_{t+1}-u}^2+\frac{\eta_t}{2}\norm{g_t}^2}\nonumber\\
=~&\frac{1}{2\eta_1}\norm{x_1-u}^2\rpar{\prod_{i=1}^{T-1}\lambda_i}+\half\sum_{t=2}^T\rpar{\frac{1}{\eta_t}-\frac{\lambda_{t-1}}{\eta_{t-1}}}\rpar{\prod_{i=t}^{T-1}\lambda_i}\norm{x_t-u}^2+\sum_{t=1}^T\frac{\eta_t}{2}\rpar{\prod_{i=t}^{T-1}\lambda_i}\norm{g_t}^2.\label{eq:halfway}
\end{align}

Now, consider the second sum on the RHS. Notice that $H_{t}=\lambda^2_{t-1}H_{t-1}+1$,
\begin{equation*}
\frac{1}{\eta_t}-\frac{\lambda_{t-1}}{\eta_{t-1}}=\frac{G}{D}\rpar{\sqrt{H_t}-\lambda_{t-1}\sqrt{H_{t-1}}}= \frac{G}{D}\rpar{\sqrt{\lambda^2_{t-1}H_{t-1}+1}-\sqrt{\lambda^2_{t-1}H_{t-1}}}\geq 0.
\end{equation*}
Therefore, we can apply the diameter condition $\norm{x_{t}-u}\leq D$ (and the Lipschitzness $\norm{g_t}\leq G$ as well), and use a telescopic sum again to obtain
\begin{align*}
\reg^{\lambda_{1:T}}_T(l_{1:T},u)&\leq\frac{D^2}{2\eta_1}\rpar{\prod_{i=1}^{T-1}\lambda_i}+\frac{D^2}{2}\sum_{t=2}^T\rpar{\frac{1}{\eta_t}-\frac{\lambda_{t-1}}{\eta_{t-1}}}\rpar{\prod_{i=t}^{T-1}\lambda_i}+\frac{G^2}{2}\sum_{t=1}^T\eta_t\rpar{\prod_{i=t}^{T-1}\lambda_i}\\
&=\frac{D^2}{2\eta_T}+\frac{G^2}{2}\sum_{t=1}^T\eta_t\rpar{\prod_{i=t}^{T-1}\lambda_i}\\
&=\frac{1}{2}DG\sqrt{H_T}+\frac{1}{2}DG\sum_{t=1}^TH_t^{-1/2}\rpar{\prod_{i=t}^{T-1}\lambda_i}.
\end{align*}

To proceed, define
\begin{equation*}
\mathrm{Sum}_T\defeq \sum_{t=1}^TH_t^{-1/2}\rpar{\prod_{i=t}^{T-1}\lambda_i}.
\end{equation*}
We now show that $\mathrm{Sum}_T\leq 2\sqrt{H_T}$ by induction. 
\begin{itemize}
\item When $T=1$, we have $H_1=1$ and $\mathrm{Sum}_1=H_1^{-1/2}=1$, therefore $\mathrm{Sum}_1\leq 2\sqrt{H_1}$.
\item When $T>1$, starting from the induction hypothesis $\mathrm{Sum}_{T-1}\leq 2\sqrt{H_{T-1}}$,
\begin{equation*}
\mathrm{Sum}_{T}=H_T^{-1/2}+\lambda_{T-1}\mathrm{Sum}_{T-1}\leq H_T^{-1/2}+2\lambda_{T-1}\sqrt{H_{T-1}}.
\end{equation*}
Applying $H_{T}=\lambda^2_{T-1}H_{T-1}+1$,
\begin{equation*}
\mathrm{Sum}_{T}\leq (\lambda^2_{T-1}H_{T-1}+1)^{-1/2}+2\sqrt{\lambda^2_{T-1}H_{T-1}}\leq 2\sqrt{\lambda^2_{T-1}H_{T-1}+1}=2\sqrt{H_T},
\end{equation*}
where the inequality is due to the concavity of square root. 
\end{itemize}
Combining everything above leads to the first part of the theorem. 

As for the second part (time-invariant $\lambda< 1$), we follow the same procedure until Eq.(\ref{eq:halfway}). The learning rate $\eta_t$ is independent of $t$; denote it as $\eta=DG^{-1}\sqrt{1-\lambda^2}$. Then, 
\begin{align*}
\reg^\lambda_T(l_{1:T},u)&\leq \frac{\lambda^{T-1}}{2\eta}\norm{x_1-u}^2+\frac{1-\lambda}{2\eta}\sum_{t=2}^T\lambda^{T-t}\norm{x_t-u}^2+\frac{\eta}{2}\sum_{t=1}^T\lambda^{T-t}\norm{g_t}^2\\
&\leq \frac{D^2\lambda^{T-1}}{2\eta}+\frac{D^2(1-\lambda)}{2\eta}\sum_{t=2}^T\lambda^{T-t}+\frac{\eta G^2}{2}\sum_{t=1}^T\lambda^{T-t}\\
&= \frac{D^2\lambda^{T-1}}{2\eta}+\frac{D^2(1-\lambda)}{2\eta}\frac{1-\lambda^{T-1}}{1-\lambda}+\frac{\eta G^2}{2}\frac{1-\lambda^{T}}{1-\lambda}\\
&\leq\frac{DG}{2\sqrt{1-\lambda^2}}+\frac{DG\sqrt{1-\lambda^2}}{2(1-\lambda)}\\
&\leq \frac{3}{2}\frac{DG}{\sqrt{1-\lambda^2}}.\tag{$1-\lambda^2\leq 2(1-\lambda)$}
\end{align*}
Following the definition of $H_T$, in the case of time-invariant $\lambda<1$,
\begin{equation*}
H_T=\frac{1-\lambda^{2T}}{1-\lambda^2}.
\end{equation*}
Due to Lemma~\ref{lemma:e}, if $T\geq \half(1-\lambda)^{-1}$, we have $\lambda^{2T}\leq e^{-1}$, therefore
\begin{equation*}
\frac{1}{\sqrt{1-\lambda^2}}\leq \sqrt{\frac{H_T}{1-e^{-1}}}.
\end{equation*}
Combining the above completes the proof.
\end{proof}

Accompanying the upper bounds, the following theorem characterizes an instance-dependent discounted regret lower bound.

\begin{theorem}\label{theorem:lower}
Consider any combination of time horizon $T\in\N_+$, discount factors $\lambda_{1:T}$, Lipschitz constant $G>0$, and variance budget $V\in\left(0,G^2H_T\right]$, where $H_T$ is defined in Eq.(\ref{eq:effective_H}). Furthermore, consider any nonzero $u\in\X$ that satisfies $-u\in\X$. For any OCO algorithm that possibly depends on these quantities, there exists a sequence of linear losses $l_t(x)=\inner{g_t}{x}$ such that $\norm{g_t}\leq G$ for all $t\in[1:T]$,
\begin{equation*}
V=\sum_{t=1}^T\rpar{\prod_{i=t}^{T-1}\lambda_i^{2}}\norm{g_t}^2,
\end{equation*}
and
\begin{equation*}
\max\spar{\reg^{\lambda_{1:T}}_T(l_{1:T},u),\reg^{\lambda_{1:T}}_T(l_{1:T},-u)}\geq \frac{1}{\sqrt{2}}\norm{u}\sqrt{V}.
\end{equation*}
\end{theorem}

\begin{proof}[Proof of Theorem~\ref{theorem:lower}]
The proof is a mild generalization of the undiscounted argument, e.g., \cite[Theorem~5.1]{orabona2023modern}. We provide it here for completeness. 

We first define a random sequence of loss gradients. Let $\eps_1,\ldots,\eps_T$ be a sequence of iid Rademacher random variables: $\eps_t$ equals $\pm 1$ with probability $\half$ each. Also, let
\begin{equation*}
L=\sqrt{\frac{V}{\sum_{t=1}^T\prod_{i=t}^{T-1}\lambda_i^{2}}}.
\end{equation*}
Then, we define the loss gradient sequence $\tilde g_{1:T}$ as $\tilde g_t=L\eps_t\frac{u}{\norm{u}}$. Notice that $L\leq G$.

Now consider the regret with respect to $\tilde l_{1:T}$, the random linear losses induced by the random gradient sequence $\tilde g_{1:T}$.
\begin{equation*}
\reg^{\lambda_{1:T}}_T\rpar{\tilde l_{1:T},u}=\sum_{t=1}^T\rpar{\prod_{i=t}^{T-1}\lambda_i}\inner{\tilde g_t}{x_t-u}=\sum_{t=1}^T\rpar{\prod_{i=t}^{T-1}\lambda_i}\frac{L}{\norm{u}}\inner{u}{x_t}\eps_t-\sum_{t=1}^T\rpar{\prod_{i=t}^{T-1}\lambda_i}L\norm{u}\eps_t.
\end{equation*}
\begin{equation*}
\reg^{\lambda_{1:T}}_T\rpar{\tilde l_{1:T},-u}=\sum_{t=1}^T\rpar{\prod_{i=t}^{T-1}\lambda_i}\inner{\tilde g_t}{x_t+u}=\sum_{t=1}^T\rpar{\prod_{i=t}^{T-1}\lambda_i}\frac{L}{\norm{u}}\inner{u}{x_t}\eps_t+\sum_{t=1}^T\rpar{\prod_{i=t}^{T-1}\lambda_i}L\norm{u}\eps_t.
\end{equation*}
Therefore, 
\begin{align*}
&\E\spar{\max\spar{\reg^{\lambda_{1:T}}_T(\tilde l_{1:T},u),\reg^{\lambda_{1:T}}_T(\tilde l_{1:T},-u)}}\\
=~&\E\spar{\sum_{t=1}^T\rpar{\prod_{i=t}^{T-1}\lambda_i}\frac{L}{\norm{u}}\inner{u}{x_t}\eps_t}+\E\spar{\max\spar{-\sum_{t=1}^T\rpar{\prod_{i=t}^{T-1}\lambda_i}L\norm{u}\eps_t, \sum_{t=1}^T\rpar{\prod_{i=t}^{T-1}\lambda_i}L\norm{u}\eps_t}}\\
=~&\sum_{t=1}^T\rpar{\prod_{i=t}^{T-1}\lambda_i}\frac{L}{\norm{u}}\E\spar{\inner{u}{x_t}\eps_t}+L\norm{u}\E\spar{\abs{\sum_{t=1}^T\rpar{\prod_{i=t}^{T-1}\lambda_i}\eps_t}}\\
=~&L\norm{u}\E\spar{\abs{\sum_{t=1}^T\rpar{\prod_{i=t}^{T-1}\lambda_i}\eps_t}}\\
\geq~& \frac{1}{\sqrt{2}}L\norm{u}\sqrt{\sum_{t=1}^T\rpar{\prod_{i=t}^{T-1}\lambda^2_i}}\tag{Khintchine inequality \cite{haagerup1981best}}\\
=~&\frac{1}{\sqrt{2}}\norm{u}\sqrt{V}.
\end{align*}
Finally, note that the above lower-bounds the expectation. There exists a loss gradient sequence $g_{1:T}$ with $g_t\in\{Lu/\norm{u},-Lu/\norm{u}\}$, such that $\max\spar{\reg^{\lambda_{1:T}}_T(l_{1:T},u),\reg^{\lambda_{1:T}}_T(l_{1:T},-u)}\geq \frac{1}{\sqrt{2}}\norm{u}\sqrt{V}$.
\end{proof}

The next theorem, presented in Section~\ref{subsection:rescaling}, analyzes gradient adaptive OGD. 

\adagrad*

\begin{proof}[Proof of Theorem~\ref{theorem:adagrad}]
As shown in \cite[Chapter~4.2]{orabona2023modern}, applying OGD with the undiscounted gradient-dependent learning rate
\begin{equation}\label{eq:adagrad_undiscounted}
\eta_t=\frac{D}{\sqrt{\sum_{i=1}^t\norm{\hat g_i}^2}}
\end{equation}
to surrogate linear losses $\inner{\hat g_t}{\cdot}$ guarantees the undiscounted regret bound
\begin{equation*}
\sum_{t=1}^T\inner{\hat g_t}{x_t-u}\leq \frac{3}{2}D\sqrt{\sum_{t=1}^T\norm{\hat g_t}^2}.
\end{equation*}

For the discounted setting, we follow the rescaling trick from Section~\ref{subsection:rescaling}. First, consider the effective prediction rule when $\hat g_t$ is defined according to Eq.(\ref{eq:surrogate_loss}).
\begin{align*}
x_{t+1}&=\Pi_\X\rpar{x_t-\frac{D}{\sqrt{\sum_{i=1}^t\norm{\hat g_i}^2}} \hat g_t}\\
&=\Pi_\X\spar{x_t-\frac{D}{\sqrt{\sum_{i=1}^t\rpar{\prod_{j=1}^{i-1}\lambda_j^{-2}}\norm{g_i}^2}} \rpar{\prod_{i=1}^{t-1}\lambda_i^{-1}}g_t}\\
&=\Pi_\X\spar{x_t-\frac{D}{\sqrt{\sum_{i=1}^t\rpar{\prod_{j=i}^{t-1}\lambda_j^{2}}\norm{g_i}^2}}g_t},
\end{align*}
which is exactly the prediction rule this theorem proposes. As for the regret bound, we have
\begin{equation*}
\sum_{t=1}^T\inner{\hat g_t}{x_t-u}\leq \frac{3}{2}D\sqrt{\sum_{t=1}^T\norm{\hat g_t}^2}=\frac{3}{2}D\sqrt{\sum_{t=1}^T\rpar{\prod_{i=1}^{t-1}\lambda_i^{-2}}\norm{g_t}^2}.
\end{equation*}
Plugging it into Eq.(\ref{eq:reduction}) and using the definition of $V_T$ from Eq.(\ref{eq:effective_V}) complete the proof.
\end{proof}

\begin{remark}[Failure of OGD]\label{remark:ogd_ftrl}
As mentioned in Section~\ref{subsection:rescaling}, one might suspect that OGD with an even better learning rate than Theorem~\ref{theorem:adagrad} (i.e., possibly using the oracle knowledge of $\norm{u}$) can further improve the regret bound to $O(\norm{u}\sqrt{V_T})$, matching the lower bound (Theorem~\ref{theorem:lower}). We now explain where this intuition could come from, and why it is misleading. 

Consider the undiscounted setting ($\lambda_t\equiv 1$). For any scaling factor $\alpha$, it is well-known that OGD with learning rate $\eta=\alpha G^{-1}T^{-1/2}$ achieves the undiscounted regret bound $O\rpar{(\alpha+\norm{u}^2\alpha^{-1})G\sqrt{T}}$ \cite[Chapter~2.1]{orabona2023modern}, which becomes $O(\norm{u}G\sqrt{T})$ if the oracle tuning $\alpha=\norm{u}$ is allowed. This might suggest that the remaining suboptimality of Theorem~\ref{theorem:adagrad} can be closed by a similar oracle tuning, i.e., setting
\begin{equation}\label{eq:adagrad_u}
\eta_t=\frac{\norm{u}}{\sqrt{V_t}}.
\end{equation}

However, such an analogy misses a key point: the aforementioned nice property of ``$\alpha$-scaled OGD'' only holds with time-invariant learning rates, and therefore, it can be found that OGD with the time-varying learning rate Eq.(\ref{eq:adagrad_u}) does not yield the $O(\norm{u}G\sqrt{V_T})$ bound we aim for, let alone the impossibility of oracle tuning. Broadly speaking, it is known (but sometimes overlooked) that OGD has certain fundamental incompatibility with time-varying learning rates, but this can be fixed by an extra regularization centered at the origin \cite{orabona2018scale,fang2022online,jacobsen2022parameter}. Such a procedure encodes the initialization into the algorithm, which essentially makes it similar to FTRL. 
\end{remark}

\subsection{Algorithm from \texorpdfstring{\cite{zhang2024improving}}{[ZYCP23]}}\label{subsection:algorithm_prior}

In this subsection we introduce the algorithm from \cite{zhang2024improving} and its guarantee. Proposed for the undiscounted setting ($\lambda_t\equiv 1$), it achieves a scale-free regret bound that simultaneously adapts to both the loss sequence and the comparator. 

Following the polar-decomposition technique from \cite{cutkosky2018black}, the main component of this algorithm is a subroutine (a standalone one dimensional OCO algorithm) that operates on the positive real line $[0,\infty)$. We present the pseudocode of this subroutine as Algorithm~\ref{alg:base_undiscounted}. Technically it is a combination of \cite[Algorithm~1 $+$ (part of) Algorithm~2]{zhang2024improving} and \cite[Algorithm~2]{cutkosky2019artificial}. Except the choice of the $\Phi$ function which uses an unusual continuous-time analysis, everything else is somewhat standard in the community. Here is an overview of the idea.

\begin{algorithm*}[ht]
\caption{(\cite[Algorithm~1 and 2]{zhang2024improving} $+$ \cite[Algorithm~2]{cutkosky2019artificial}) Undiscounted 1D magnitude learner on $[0,\infty)$.\label{alg:base_undiscounted}}
\begin{algorithmic}[1]
\REQUIRE Hyperparameter $\eps>0$ (default $\eps=1$). 
\STATE Initialize parameters $v_1=0$, $s_1=0$, $h_1=0$. Define the following $(h,\eps)$-parameterized \emph{potential function} $\Phi_{h,\eps}(v,s)$ with two input arguments $v$ and $s$.
\begin{equation*}
\Phi_{h,\eps}(v,s)\defeq\eps\sqrt{v+2hs+16h^2}\rpar{2\int_0^{\frac{s}{2\sqrt{v+2hs+16h^2}}}\erfi(u)du-1}.
\end{equation*}
\FOR{$t=1,2,\ldots$}
\STATE If $h_t=0$, define an unprojected prediction $\tilde x_t=0$; otherwise, define
\begin{align*}
\tilde x_t&=\partial_2\Phi_{h_{t},\eps}(v_{t},s_{t})\\
&=\eps\cdot\erfi\rpar{\frac{s_{t}}{2\sqrt{v_{t}+2h_{t}s_{t}+16h^2_{t}}}}-\frac{\eps h_{t}}{\sqrt{v_t+2h_ts_t+16h_t^2}}\exp\spar{\frac{s_t^2}{4(v_t+2h_ts_t+16h_t^2)}}.
\end{align*}
\STATE Predict $x_t=\Pi_{[0,\infty)}\rpar{\tilde x_t}$, the projection of $\tilde x_t$ to the domain $[0,\infty)$.
\STATE Receive the loss gradient $g_t\in\R$.
\STATE Clip the gradient by defining $g_{t,\mathrm{clip}}=\Pi_{[-h_t,h_t]}\rpar{g_t}$, and let $h_{t+1}=\max\rpar{h_t,\abs{g_t}}$.
\STATE \label{line:projection_surrogate_gradient}Define a surrogate loss gradient $\tilde g_{t,\clip}\in\R$, where
\begin{equation*}
\tilde g_{t,\mathrm{clip}}=\begin{cases}
g_{t,\mathrm{clip}},& g_{t,\mathrm{clip}}\tilde x_t\geq g_{t,\mathrm{clip}}x_t,\\
0,& \textrm{else}.
\end{cases}
\end{equation*}
\STATE Update $v_{t+1}=v_{t}+\tilde g^2_{t,\mathrm{clip}}$, $s_{t+1}=s_{t}-\tilde g_{t,\mathrm{clip}}$.
\ENDFOR
\end{algorithmic}
\end{algorithm*}

\begin{itemize}
\item The core component is the \emph{potential function} $\Phi_{h,\eps}(v,s)$ that takes two input arguments, the \emph{observed gradient variance} $v$ and the \emph{observed gradient sum} $s$. For now let us briefly suppress the dependence on auxiliary parameters $h$ and $\eps$. At the beginning of the $t$-th round, with the observations of past gradients $g_{1:t-1}$, we ``ideally'' (there are two twists explained later) would like to define its summaries (sometimes called \emph{sufficient statistics})
\begin{equation*}
v_t=\sum_{i=1}^{t-1}g^2_i,\quad s_t=-\sum_{i=1}^{t-1}g_i,
\end{equation*}
and predict $x_t=\partial_2\Phi(v_t,s_t)$, i.e., the partial derivative of the potential function $\Phi$ with respect to its second argument, evaluated at the pair $(v_t,s_t)$.\footnote{In Section~\ref{subsection:rescaling}, we call the dependence of $\partial_2\Phi(v_t,s_t)$ on $s_t$ as the \emph{prediction function}.} This is the standard procedure of the ``potential method'' \cite{cesa2006prediction,mcmahan2014unconstrained,mhammedi2020lipschitz} in online learning, which is associated to a well-established analytical strategy \cite{mcmahan2014unconstrained}. Equivalently, one may also interpret this procedure as \emph{Follow the Regularized Leader} (FTRL) \cite[Chapter~7.3]{orabona2023modern} with linearized losses: the potential function $\Phi$ is essentially the convex conjugate of a FTRL regularizer \cite[Section~3.1]{zhang2022pde}.

A bit more on the design of $\Phi$: the intuition is that $h$ is typically much smaller than $v$ and $s$, so if we set $h=0$ (rigorously this is not allowed since the prediction would be a bit too aggressive; but just for our intuitive discussion this is fine), then the potential function is morally
\begin{equation*}
\Phi(v,s)\approx\eps\sqrt{v}\rpar{2\int_0^{\frac{s}{2\sqrt{v}}}\erfi(u)du-1},
\end{equation*}
which is associated to the prediction function
\begin{equation*}
\partial_2\Phi(v,s)\approx\eps\cdot\erfi\rpar{\frac{s}{2\sqrt{v}}}=\eps\int_0^{s/\sqrt{4v}}\exp(u^2)du.
\end{equation*}
The use of the $\erfi$ function might seem obscure, but essentially it is rooted in a recent trend \cite{drenska2020prediction,zhang2022pde,harvey2023optimal} connecting online learning to stochastic calculus: we scale the OCO game towards its continuous-time limit and solve the obtained \emph{Backward Heat Equation}. Prior to this trend, the predominant idea was using \cite{mcmahan2014unconstrained,orabona2016coin,mhammedi2020lipschitz}
\begin{equation}\label{eq:classical}
\Phi(v,s)\approx\frac{\eps}{\sqrt{v}}\exp\rpar{\frac{s^2}{\mathrm{constant}\cdot v}},
\end{equation}
\begin{equation*}
\partial_2\Phi(v,s)\approx\eps\frac{\mathrm{constant}\cdot s}{v^{3/2}}\exp\rpar{\frac{s^2}{\mathrm{constant}\cdot v}}.
\end{equation*}
It can be shown that the $\erfi$ prediction rule is quantitatively stronger, and quite importantly, it makes the hyperparameter $\eps$ ``unitless'' \cite[Section~5]{zhang2024improving}. In contrast, in the more classical potential function Eq.(\ref{eq:classical}), $\eps$ carries the unit of ``gradient squared'', which means the algorithm requires a guess of the \emph{time-uniform Lipschitz constant} $\max_t\norm{g_t}$ at the \emph{beginning of the game}. Due to the discussion in Section~\ref{subsection:rescaling} (right after introducing the rescaling trick), this suffers from certain suboptimality (Remark~\ref{remark:benefit}) when applied with rescaling. 

\item Although most of the heavy lifting is handled by the choice of $\Phi$, a remaining issue is that with $h\neq 0$, the prediction $x_t=\partial_2\Phi(v_t,s_t)$ is not necessarily positive, which violates the domain constraint $[0,\infty)$. To fix this issue, we adopt the technique from \cite{cutkosky2018black,cutkosky2020parameter} which has become a standard tool for the community: predict $x_t=\Pi_{[0,\infty)}(\tilde x_t)$ which is the projection of $\tilde x_t=\partial_2\Phi(v_t,s_t)$ to the domain $[0,\infty)$, and update the sufficient statistics $(v_{t+1},s_{t+1})$ using a \emph{surrogate loss gradient} $\tilde g_{t,\clip}$ (Line~\ref{line:projection_surrogate_gradient} of Algorithm~\ref{alg:base_undiscounted}) instead of $g_{t,\clip}$ (the clipping will be explained later; for now we might regard $g_{t,\clip}=g_t$, the actual gradient).

The intuition behind Line~\ref{line:projection_surrogate_gradient} is that, if the unprojected prediction $\tilde x_t=\partial_2\Phi(v_t,s_t)$ is already negative and the actual loss gradient $g_{t,\clip}$ encourages it to be ``more negative'', then we set $\tilde g_{t,\clip}=0$ in the update of the $(v_{t+1},s_{t+1})$ pair to avoid this undesirable behavior (unprojected prediction drifting away from the domain). In other situations, it is fine to use $g_{t,\clip}$ directly in the update of $(v_{t+1},s_{t+1})$, therefore we simply set $\tilde g_{t,\clip}=g_{t,\clip}$.

Rigorously, \cite[Theorem~2]{cutkosky2020parameter} shows that for all comparator $u\in[0,\infty)$, $\inner{g_{t,\clip}}{x_t-u}\leq \inner{\tilde g_{t,\clip}}{\tilde x_t-u}$. That is, as long as the unprojected prediction sequence $\tilde x_{1:T}$ guarantees a good regret bound (in an improper manner, i.e., violating the domain constraint), then the projected prediction sequence $x_{1:T}$ also guarantees a good regret bound, but properly.

\item Another twist is related to \emph{updating} the auxiliary parameter $h$, which was previously ignored. Due to the typical limitation of FTRL algorithms, we have to guess the range of $g_t$ (i.e., a time-varying Lipschitz constant $G_t$ such that $\norm{g_t}\leq G_t$) \emph{right before} making the prediction $x_t$, and $h_t$ serves as this guess. \cite{cutkosky2019artificial} suggests using the range of past loss gradients $h_t=\max_{i\in[1:t-1]}\abs{g_i}$ to guess that $\abs{g_t}\leq h_t$. Surely this could be wrong: in that case ($\abs{g_t}>h_t$), we clip $g_t$ to $[-h_t,h_t]$ before sending it to the $(v_{t+1},s_{t+1})$ update. 

We also clarify a possible confusion related to ``guessing the Lipschitzness''. We have argued that if an algorithm requires guessing the time-uniform Lipschitz constant $\max_t\norm{g_t}$ at the beginning of the game, then it does not serve as a good base algorithm $\A$ in our rescaling trick. The use of $h_t$ is different: it is updated online as a guess of $\norm{g_t}$, which is fine (to apply with rescaling). 
\end{itemize}

In terms of the concrete guarantee, the following theorem characterizes the undiscounted regret bound of Algorithm~\ref{alg:base_undiscounted}. The proof is a straightforward corollary of \cite[Lemma~B.2]{zhang2024improving} and \cite[Theorem~2]{cutkosky2020parameter}, therefore omitted.

\begin{theorem}[\cite{cutkosky2020parameter,zhang2024improving}]\label{theorem:base_undiscounted}
Algorithm~\ref{alg:base_undiscounted} guarantees for all time horizon $T\in\N_+$, loss gradients $g_{1:T}$ and comparator $u\in[0,\infty)$,
\begin{equation*}
\sum_{t=1}^Tg_t(x_t-u)\leq \eps\sqrt{\sum_{t=1}^Tg^2_t+2GS+16G^2}+uS+\sum_{t=1}^T\abs{g_t-g_{t,\mathrm{clip}}}\abs{x_t-u},
\end{equation*}
where $G=\max_{t\in[1:T]}\abs{g_t}$ and
\begin{equation*}
S= 8G\rpar{1+\sqrt{\log(2 u\eps^{-1}+1)}}^2+2\sqrt{\sum_{t=1}^Tg^2_t+16G^2}\rpar{1+\sqrt{\log(2u\eps^{-1}+1)}}.
\end{equation*}
\end{theorem}

\begin{remark}[Iterate stability]\label{remark:stability}
The above bound has an iterate stability term $\sum_{t=1}^T\abs{g_t-g_{t,\mathrm{clip}}}\abs{x_t-u}$, which can be further upper-bounded by $(u+\max_{t\in[1:T]}x_t)G_T$. Similar characterizations of stability have appeared broadly in stochastic optimization before \cite{orabona2021parameter,ivgi2023dog,orabona2023normalized}. For the general adversarial setting we consider, \cite{cutkosky2019artificial} suggests using \emph{artificial constraints} to trade $\max_tx_t$ for terms that only depend on $g_{1:T}$ and $u$. We do not take this route because ($i$) it makes our discounted algorithm more complicated but does not seem to improve the performance; and ($ii$) even without artificial constraints, the prediction magnitude is indeed controlled in our main application (online conformal prediction), and possibly in the downstream stochastic setting as well (following \cite{orabona2023normalized}). 
\end{remark}

Given the above undiscounted 1D subroutine, we can extend it to $\R^d$ following the standard polar-decomposition technique \cite{cutkosky2018black}. Overall, the algorithm becomes the $\lambda_t\equiv 1$ special case of Algorithm~\ref{alg:meta} (our main algorithm presented in Section~\ref{subsection:algorithm_detail}). This is as expected, since the discounted setting is a strict generalization. The following theorem is essentially \cite[Theorem~2 $+$ the discussion after that]{zhang2024improving}. 

\begin{theorem}[\cite{zhang2024improving}]\label{theorem:main_undiscounted}
With $\lambda_t=1$ for all $t$, Algorithm~\ref{alg:meta} from Section~\ref{subsection:algorithm_detail} guarantees for all $T\in\N_+$, loss gradients $g_{1:T}$ and comparator $u\in\R^d$,
\begin{equation*}
\sum_{t=1}^T\inner{g_t}{x_t-u}\leq O\rpar{\norm{u}\sqrt{V_T\log(\norm{u}\eps^{-1})}\vee\norm{u}G_T\log(\norm{u}\eps^{-1})}+\sum_{t=1}^T\norm{g_t-g_{t,\mathrm{clip}}}\norm{x_t-u},
\end{equation*}
where $V_T$ and $G_T$ are defined in Eq.(\ref{eq:effective_V}), and $O(\cdot)$ is in the regime of large $V_T$ $(V_T\gg G_T)$ and large $\norm{u}$ $(\norm{u}\gg \eps)$. Furthermore, if the comparator $u=0$, we have
\begin{equation*}
\sum_{t=1}^T\inner{g_t}{x_t}\leq O\rpar{\eps\sqrt{V_T}}+\sum_{t=1}^T\norm{g_t-g_{t,\mathrm{clip}}}\norm{x_t}.
\end{equation*}
\end{theorem}

\subsection{Analysis of the main algorithm}\label{subsection:algorithm_detail}

This subsection presents our main discounted algorithm and its regret bound. We start from its 1D magnitude learner. 

\base*

\begin{proof}[Proof of Theorem~\ref{theorem:base}] The proof mostly follows from carefully checking the equivalence of the following two algorithms: ($i$) Algorithm~\ref{alg:base} (the discounted algorithm) on a sequence of loss gradients $g_{1:T}$, and ($ii$) Algorithm~\ref{alg:base_undiscounted} (the undiscounted algorithm) on the sequence of scaled surrogate gradients, $\rpar{\prod_{i=1}^{t-1}\lambda_i^{-1}}g_t;\forall t\in[1:T]$. Since the quantities in Algorithm~\ref{alg:base} and \ref{alg:base_undiscounted} follow the same notation, we separate them by adding a superscript $D$ on quantities in Algorithm~\ref{alg:base}, and $ND$ in their Algorithm~\ref{alg:base_undiscounted} counterparts.

We show this by induction: suppose for some $t\in[1:T]$, we have
\begin{equation*}
s^D_t=\rpar{\prod_{i=1}^{t-2}\lambda_i}s^{ND}_t,\quad v^D_t=\rpar{\prod_{i=1}^{t-2}\lambda^2_i}v^{ND}_t,\quad h^D_t=\rpar{\prod_{i=1}^{t-2}\lambda_i}h^{ND}_t.
\end{equation*}
Such an induction hypothesis holds for $t=1$. Then, from the prediction rules (and the fact that all the discount factors are strictly positive), the unprojected predictions $\tilde x^D_t=\tilde x^{ND}_t$, and the projected predictions $x^D_t=x^{ND}_t$.

Now consider the gradient clipping. Since $g_t^{ND}=\rpar{\prod_{i=1}^{t-1}\lambda_i^{-1}}g^D_t$, 
\begin{align*}
g^{ND}_{t,\mathrm{clip}}&=\Pi_{[-h^{ND}_t,h^{ND}_t]}\rpar{g^{ND}_t}\\
&=c^{-1}\Pi_{[-ch^{ND}_t,ch^{ND}_t]}\rpar{cg^{ND}_t}\tag{$\forall c>0$}\\
&=\rpar{\prod_{i=1}^{t-1}\lambda^{-1}_i}\Pi_{[-\lambda_{t-1}h^{D}_t,\lambda_{t-1}h^{D}_t]}\rpar{g^{D}_t}\tag{$c=\prod_{i=1}^{t-1}\lambda_i$}\\
&=\rpar{\prod_{i=1}^{t-1}\lambda^{-1}_i}g^{D}_{t,\mathrm{clip}}.
\end{align*}
Then, due to $x^D_t=x^{ND}_t$ and $\tilde x^D_t=\tilde x^{ND}_t$, we have $\tilde g_{t,\clip}^{ND}=\rpar{\prod_{i=1}^{t-1}\lambda_i^{-1}}\tilde g^D_{t,\clip}$.

Finally, consider the updates of $s^D_{t+1}$, $v^D_{t+1}$ and $h^D_{t+1}$, as well as their Algorithm~\ref{alg:base_undiscounted} counterparts. 
\begin{equation*}
s^D_{t+1}=\lambda_{t-1}s^D_t-\tilde g^{D}_{t,\mathrm{clip}}=\rpar{\prod_{i=1}^{t-1}\lambda_i}s^{ND}_t-\rpar{\prod_{i=1}^{t-1}\lambda_i}\tilde g^{ND}_{t,\mathrm{clip}}=\rpar{\prod_{i=1}^{t-1}\lambda_i}\rpar{s^{ND}_t-\tilde g^{ND}_{t,\mathrm{clip}}}=\rpar{\prod_{i=1}^{t-1}\lambda_i}s^{ND}_{t+1}.
\end{equation*}
Similarly, 
\begin{equation*}
v^D_{t+1}=\rpar{\prod_{i=1}^{t-1}\lambda^2_i}v^{ND}_{t+1},
\end{equation*}
\begin{equation*}
h^D_{t+1}=\max\rpar{\lambda_{t-1}h^D_t,\abs{ g^D_t}}=\max\spar{\rpar{\prod_{i=1}^{t-1}\lambda_i}h^{ND}_t,\rpar{\prod_{i=1}^{t-1}\lambda_i}\abs{g^{ND}_t}}=\rpar{\prod_{i=1}^{t-1}\lambda_i}h^{ND}_{t+1}.
\end{equation*}
That is, the induction hypothesis holds for $t+1$, and therefore we have shown the equivalence of the considered two algorithms. 

As for the regret bound of Algorithm~\ref{alg:base}, combining the reduction Eq.(\ref{eq:reduction}) and the regret bound of Algorithm~\ref{alg:base_undiscounted} (Theorem~\ref{theorem:base_undiscounted} from Appendix~\ref{subsection:algorithm_prior}) immediately gives us
\begin{equation*}
\reg^{\lambda_{1:T}}_T(l_{1:T},u)\leq \eps\sqrt{V_T+2G_TS+16G_T^2}+uS+\rpar{\prod_{t=1}^{T-1}\lambda_t}\sum_{t=1}^T\rpar{\prod_{i=1}^{t-1}\lambda^{-1}_i}\abs{g_t-g_{t,\mathrm{clip}}}\abs{x_t-u},
\end{equation*}
where
\begin{equation*}
S= 8G_T\rpar{1+\sqrt{\log(2 u\eps^{-1}+1)}}^2+2\sqrt{V_T+16G_T^2}\rpar{1+\sqrt{\log(2u\eps^{-1}+1)}}.
\end{equation*}
The remaining clipping error term can be bounded similarly as \cite[Theorem~2]{cutkosky2019artificial}. For any $\tau\in[1:T]$,
\begin{align*}
\sum_{t=1}^{T-\tau}\rpar{\prod_{i=1}^{t-1}\lambda^{-1}_i}\abs{g_t-g_{t,\mathrm{clip}}}\abs{x_t-u}&\leq \max_{t\in[1:T-\tau]}\abs{x_t-u}\sum_{t=1}^{T-\tau}\rpar{\prod_{i=1}^{t-1}\lambda^{-1}_i}\abs{g_t-g_{t,\mathrm{clip}}}\\
&= \max_{t\in[1:T-\tau]}\abs{x_t-u}\sum_{t=1}^{T-\tau}\rpar{\prod_{i=1}^{t-1}\lambda^{-1}_i}\rpar{h_{t+1}-\lambda_{t-1}h_t}\tag{from Algorithm~\ref{alg:base}}\\
&= \max_{t\in[1:T-\tau]}\abs{x_t-u}\sum_{t=1}^{T-\tau}\spar{\rpar{\prod_{i=1}^{t-1}\lambda^{-1}_i}h_{t+1}-\rpar{\prod_{i=1}^{t-2}\lambda^{-1}_i}h_t}\\
&= \max_{t\in[1:T-\tau]}\abs{x_t-u}\spar{\rpar{\prod_{i=1}^{T-\tau-1}\lambda^{-1}_i}h_{T-\tau+1}-h_1}\\
&\leq \rpar{\max_{t\in[1:T-\tau]}x_t+u}\rpar{\prod_{i=1}^{T-\tau-1}\lambda^{-1}_i}h_{T-\tau+1}.
\end{align*}
Similarly, 
\begin{equation*}
\sum_{t=T-\tau+1}^T\rpar{\prod_{i=1}^{t-1}\lambda^{-1}_i}\abs{g_t-g_{t,\mathrm{clip}}}\abs{x_t-u}\leq \rpar{\max_{t\in[T-\tau+1:T]}x_t+u}\spar{\rpar{\prod_{i=1}^{T-1}\lambda^{-1}_i}h_{T+1}-\rpar{\prod_{i=1}^{T-\tau-1}\lambda^{-1}_i}h_{T-\tau+1}}.
\end{equation*}
Finally, multiplying $\prod_{t=1}^{T-1}\lambda_t$ and using $G_T=h_{T+1}$ and $G_{T-\tau}=h_{T-\tau+1}$ complete the proof.
\end{proof}

As for the extension to $\R^d$ following \cite{cutkosky2018black}, we present the pseudocode as Algorithm~\ref{alg:meta}. There is a small twist: when applying Algorithm~\ref{alg:base} as the 1D subroutine, in its Line~\ref{line:clip} we set $g_{t,\clip}=g_t$, and $h_{t+1}$ is given by the meta-algorithm. That is, the gradient clipping is handled on the high level (Algorithm~\ref{alg:meta}) rather than the low level (Algorithm~\ref{alg:base}).

\begin{algorithm*}[ht]
\caption{Discounted adaptivity on $\R^d$.\label{alg:meta}}
\begin{algorithmic}[1]
\STATE Define $\A_{1d}$ as a minor variant of Algorithm~\ref{alg:base} (with hyperparameter $\eps$), where its Line~\ref{line:clip} is replaced by: ``Set $g_{t,\clip}=g_t$, and receive a hint $h_{t+1}$.''
\STATE Define $\A_\ball$ as the algorithm from Theorem~\ref{theorem:adagrad}, on the $d$-dimensional unit $L_2$ norm ball (with $D=2$).
\STATE Initialize $h_1=0$.
\FOR{$t=1,2,\ldots$}
\STATE Query $\A_{1d}$ for its prediction $y_t\in\R$. 
\STATE Query $\A_\ball$ for its prediction $w_t\in\R^d$; $\norm{w_t}\leq 1$. 
\STATE Predict $x_t=w_ty_t$, receive the loss gradient $g_t\in\R^d$ and the discount factor $\lambda_t\in(0,\infty)$.
\STATE Update $h_{t+1}=\max\rpar{\lambda_{t-1}h_t,\norm{g_t}}$, and define $g_{t,\clip}=g_t\lambda_{t-1}h_t/h_{t+1}$ (if $h_{t+1}=0$ then $g_{t,\clip}=0$).
\STATE Send $\inner{g_{t,\clip}}{w_t}$ and $g_{t,\clip}$ as the surrogate loss gradients to $\A_{1d}$ and $\A_\ball$, respectively.
\STATE Send the discount factor $\lambda_t$ to $\A_{1d}$ and $\A_\ball$.
\STATE Send the hint $h_{t+1}$ to $\A_{1d}$.
\ENDFOR
\end{algorithmic}
\end{algorithm*}

Algorithm~\ref{alg:meta} induces our main theorem (Theorem~\ref{theorem:main}). The proof combines the undiscounted regret bound (Theorem~\ref{theorem:main_undiscounted}) and our rescaling trick. It is almost the same as the above proof of Theorem~\ref{theorem:base}, therefore omitted.

\begin{remark}[Benefit of \cite{zhang2024improving}]\label{remark:benefit}
We used \cite{zhang2024improving} as the base algorithm in our scaling trick, but there are other options \cite{mhammedi2020lipschitz,jacobsen2022parameter}. The problem of such alternatives is that, they still need an estimate of the time-uniform Lipschitz constant at the beginning of the game, due to certain unit inconsistency (discussed in Appendix~\ref{subsection:algorithm_prior}, e.g., Eq.(\ref{eq:classical})). In the undiscounted setting, they guarantee
\begin{equation*}
\sum_{t=1}^T\inner{g_t}{x_t-u}\leq O\rpar{\norm{u}\sqrt{V_T\log(\norm{u}T)}\vee\norm{u}G_T\log(\norm{u}T)}+\sum_{t=1}^T\norm{g_t-g_{t,\mathrm{clip}}}\norm{x_t-u},
\end{equation*}
as opposed to Theorem~\ref{theorem:main_undiscounted} in this paper. After the scaling trick, the $T$ dependence in this bound will be transferred to the obtained discounted regret bound, such that the latter also depends on the end time $T$. In other words, such a discounted regret bound gradually degrades over time. 
\end{remark}

\section{Detail of Section~\ref{section:conformal}}\label{section:conformal_detail}

This section presents omitted details of our OCP application. First, we present the pseudocode of $\A_{CP}$ in Appendix~\ref{subsection:ocp_subroutine}, with the notations (relevant quantities are with the superscipt $*$) from the OCP setting. All the concrete proofs are presented in Appendix~\ref{subsection:ocp_proof}. 

\subsection{Pseudocode of OCP algorithm}\label{subsection:ocp_subroutine}

The pseudocode of $\A_{CP}$ is Algorithm~\ref{alg:ocp}. This is equivalent to directly applying Algorithm~\ref{alg:base}, our main 1D OCO algorithm, to the setting of Section~\ref{section:conformal}. 

\begin{algorithm*}[ht]
\caption{The proposed OCP algorithm $\A_{CP}$.\label{alg:ocp}}
\begin{algorithmic}[1]
\REQUIRE Hyperparameter $\eps>0$ (default $\eps=1$). 
\STATE Initialize $S^*_{0,\clip}=0$, $V^*_{0,\clip}=0$, $G^*_0=0$. 
\FOR{$t=1,2,\ldots$}
\STATE If $G^*_{t-1}=0$, define the unprojected prediction $\tilde r_t=0$. Otherwise,
\begin{multline*}
\tilde r_t=\eps\cdot\erfi\rpar{\frac{S^*_{t-1,\clip}}{2\sqrt{V^*_{t-1,\clip}+2G^*_{t-1}S^*_{t-1,\clip}+16\rpar{G^*_{t-1}}^2}}}\\
-\frac{\eps G^*_{t-1}}{\sqrt{V^*_{t-1,\clip}+2G^*_{t-1}S^*_{t-1,\clip}+16\rpar{G^*_{t-1}}^2}}\exp\spar{\frac{\rpar{S^*_{t-1,\clip}}^2}{4\rpar{V^*_{t-1,\clip}+2G^*_{t-1}S^*_{t-1,\clip}+16\rpar{G^*_{t-1}}^2}}}.
\end{multline*}
\STATE Predict $r_t=\Pi_{[0,\infty)}\rpar{\tilde r_t}$.
\STATE Receive the OCP loss gradient $g^*_t\in\R$ and the discount factor $\lambda_{t-1}\in(0,\infty)$.
\STATE Clip $g^*_t$ by defining
\begin{equation*}
g^*_{t,\mathrm{clip}}=\Pi_{[-\lambda_{t-1}G^*_{t-1},\lambda_{t-1}G^*_{t-1}]}\rpar{g^*_t}.
\end{equation*}

\STATE Compute running statistics of past observations,
\begin{equation*}
S^*_{t,\clip}=-\sum_{i=1}^t\rpar{\prod_{j=i}^{t-1}\lambda_j}g^*_{i,\clip},\quad
V^*_{t,\clip}=\sum_{i=1}^t\rpar{\prod_{j=i}^{t-1}\lambda_j^{2}}\norm{g^*_{i,\clip}}^2, \quad
G^*_t=\max_{i\in[1:t]}\rpar{\prod_{j=i}^{t-1}\lambda_j}\norm{g^*_i}.
\end{equation*}
\ENDFOR
\end{algorithmic}
\end{algorithm*}

In particular, the problem structure of OCP allows removing the surrogate loss construction (Line~\ref{line:projection} of Algorithm~\ref{alg:base}), which makes the algorithm slightly simpler. To see this, notice that the surrogate loss $\tilde g_{t,\clip}$ there is only needed (i.e., does not equal $g_{t,\clip}$) when the unprojected prediction $\tilde x_t<0$ and the projected prediction $x_t=0$. In the OCP notation, this means that our radius prediction $r_t=0\leq r^*_t$, therefore the subgradient $g^*_t$ evaluated at $r_t$ satisfies $g^*_t\leq 0$. Back to the notation of Algorithm~\ref{alg:base}, we have $g_t\leq 0$, therefore after clipping $g_{t,\clip}\leq 0$. Putting things together, 
\begin{equation*}
g_{t,\clip}\rpar{\tilde x_t-x_t}\geq 0. 
\end{equation*}
That is, the condition in Algorithm~\ref{alg:base} that triggers $\tilde g_{t,\clip}=0$ is impossible, therefore $\tilde g_{t,\clip}$ always equals $g_{t,\clip}$.

\subsection{Omitted proofs}\label{subsection:ocp_proof}

Moving to the analysis, we first prove the key lemma connecting the prediction magnitude of $\A_{CP}$ to the coverage metric $S^*_t$, Eq.(\ref{eq:discounted_g_sum}). This is divided into two steps. 
\begin{itemize}
\item First (Lemma~\ref{lemma:connecting_basic}), we approximate the prediction rule of $\A_{CP}$, such that $r_{t+1}$ characterizes the discounted sum of \emph{clipped gradients}
\begin{equation}\label{eq:discounted_g_sum_clip}
S^*_{t,\clip}=-\sum_{i=1}^t\rpar{\prod_{j=i}^{t-1}\lambda_j}g^*_{i,\clip},
\end{equation}
which is an internal quantity of Algorithm~\ref{alg:ocp}. 

\item The second step (Lemma~\ref{lemma:connecting}) is connecting $S^*_{t,\clip}$ to the unclipped version $S^*_t$, Eq.(\ref{eq:discounted_g_sum}). This is the version presented in the main paper.
\end{itemize}

\begin{lemma}\label{lemma:connecting_basic}
Consider $S^*_{t,\clip}$ from Eq.(\ref{eq:discounted_g_sum_clip}). $\A_{CP}$ (Algorithm~\ref{alg:ocp}) guarantees for all $t$,
\begin{equation*}
\abs{S^*_{t,\clip}}\leq 2\sqrt{V^*_{t,\clip}}\rpar{1+\sqrt{\log\rpar{1+2r_{t+1}\eps^{-1}}}}+13G^*_t\rpar{1+\sqrt{\log\rpar{1+2r_{t+1}\eps^{-1}}}}^2,
\end{equation*}
where $V^*_{t,\clip}$ and $G^*_t$ are internal quantities of Algorithm~\ref{alg:ocp}.
\end{lemma}

\begin{proof}[Proof of Lemma~\ref{lemma:connecting_basic}]
Throughout this proof we will consider the internal variables of Algorithm~\ref{alg:ocp}. The superscript $*$ are always removed to simplify the notation. Assume without loss of generality that internally in Algorithm~\ref{alg:ocp}, $G_{t}\neq 0$ for the considered $t$. Otherwise, all the gradients $g_1,\ldots,g_t$ are zero, which makes the statement of the lemma obvious. 

To upper-bound $S_{t,\clip}$, the analysis is somewhat similar to the control of Fenchel conjugate in \cite[Lemma~B.1]{zhang2024improving}, although the latter serves a different purpose. Consider the input argument of the $\erfi$ function in the definition of $r_{t+1}$. There are two cases. 

\begin{itemize}
\item \textbf{Case 1: $S_{t,\clip}< 2\sqrt{V_{t,\clip}+2G_tS_{t,\clip}+16G^2_t}$.}

With straightforward algebra, 
\begin{equation*}
S^2_{t,\clip}-8G_tS_{t,\clip}-4\rpar{V_{t,\clip}+16G^2_t}< 0,
\end{equation*}
\begin{equation*}
S_{t,\clip}\leq 4G_t+\sqrt{16G^2_t+4\rpar{V_{t,\clip}+16G_t^2}}\leq 2\sqrt{V_{t,\clip}}+13G_t.
\end{equation*}

\item \textbf{Case 2: $S_{t,\clip}\geq 2\sqrt{V_{t,\clip}+2G_tS_{t,\clip}+16G^2_t}$.}

Since $G_t\neq 0$, Algorithm~\ref{alg:ocp} predicts $r_{t+1}=\Pi_{[0,\infty)}(\tilde r_{t+1})$, where
\begin{align*}
\tilde r_{t+1}&=\eps\cdot\erfi\rpar{\frac{S_{t,\clip}}{2\sqrt{V_{t,\clip}+2G_tS_{t,\clip}+16G^2_t}}}\\
&\quad\quad-\frac{\eps G_t}{\sqrt{V_{t,\clip}+2G_tS_{t,\clip}+16G_t^2}}\exp\spar{\frac{S_{t,\clip}^2}{4(V_{t,\clip}+2G_tS_{t,\clip}+16G_t^2)}}.
\end{align*}

Due to a lower estimate of the $\erfi$ function \cite[Lemma~A.3]{zhang2024improving}, for all $x\geq 1$, $\erfi(x)\geq \exp(x^2)/2x$. Then, 
\begin{equation*}
\erfi\rpar{\frac{S_{t,\clip}}{2\sqrt{V_{t,\clip}+2G_tS_{t,\clip}+16G^2_t}}}\geq \frac{\sqrt{V_{t,\clip}+2G_tS_{t,\clip}+16G^2_t}}{S_{t,\clip}}\exp\spar{\frac{S_{t,\clip}^2}{4(V_{t,\clip}+2G_tS_{t,\clip}+16G_t^2)}},
\end{equation*}
and simple algebra characterizes the multiplier on the RHS, 
\begin{align*}
&\frac{\sqrt{V_{t,\clip}+2G_tS_{t,\clip}+16G^2_t}}{S_{t,\clip}}-\frac{2G_t}{\sqrt{V_{t,\clip}+2G_tS_{t,\clip}+16G_t^2}}\\
=~&\frac{V_{t,\clip}+16G_t^2}{S_{t,\clip}\sqrt{V_{t,\clip}+2G_tS_{t,\clip}+16G_t^2}}\geq 0. 
\end{align*}
Therefore, 
\begin{align*}
r_{t+1}&\geq\tilde r_{t+1}\\
&\geq \frac{\eps}{2}\erfi\rpar{\frac{S_{t,\clip}}{2\sqrt{V_{t,\clip}+2G_tS_{t,\clip}+16G^2_t}}}+\eps\exp\spar{\frac{S_{t,\clip}^2}{4(V_{t,\clip}+2G_tS_{t,\clip}+16G_t^2)}}\\
&\quad\quad\times\rpar{\frac{\sqrt{V_{t,\clip}+2G_tS_{t,\clip}+16G^2_t}}{2S_{t,\clip}}-\frac{G_t}{\sqrt{V_{t,\clip}+2G_tS_{t,\clip}+16G_t^2}}}\\
&\geq\frac{\eps}{2}\erfi\rpar{\frac{S_{t,\clip}}{2\sqrt{V_{t,\clip}+2G_tS_{t,\clip}+16G^2_t}}}.
\end{align*}

Due to another estimate of $\erfi^{-1}$ \cite[Lemma~A.4]{zhang2024improving}, for all $x\geq 0$ we have $\erfi^{-1}(x)\leq 1+\sqrt{\log(x+1)}$. Then, 
\begin{align*}
S_{t,\clip}&\leq 2\sqrt{V_{t,\clip}+2G_tS_{t,\clip}+16G_t^2}\cdot\erfi^{-1}\rpar{2r_{t+1}\eps^{-1}}\\
&\leq 2\sqrt{V_{t,\clip}+2G_tS_{t,\clip}+16G_t^2}\rpar{1+\sqrt{\log\rpar{1+2r_{t+1}\eps^{-1}}}}.
\end{align*}
Similar to the algebra of Case 1, 
\begin{equation*}
S_{t,\clip}\leq 2\sqrt{V_{t,\clip}}\rpar{1+\sqrt{\log\rpar{1+2r_{t+1}\eps^{-1}}}}+13G_t\rpar{1+\sqrt{\log\rpar{1+2r_{t+1}\eps^{-1}}}}^2.
\end{equation*}
\end{itemize}

Combining the two cases, we have
\begin{equation*}
S_{t,\clip}\leq 2\sqrt{V_{t,\clip}}\rpar{1+\sqrt{\log\rpar{1+2r_{t+1}\eps^{-1}}}}+13G_t\rpar{1+\sqrt{\log\rpar{1+2r_{t+1}\eps^{-1}}}}^2.
\end{equation*}
That is, $S_{t,\clip}$ is now bounded from the above by $\tilde O(\sqrt{V_{t,\clip}})$. 

On the other side, we now consider bounding $S_{t,\clip}$ from below. This is a classical induction argument similar to \cite[Lemma~4.1]{zhang2024improving}. Consider any time index $i\in[1:t]$, 
\begin{itemize}
\item If $S_{i-1,\clip}\leq 0$, then according to the prediction rule of Algorithm~\ref{alg:ocp}, we have $r_{i}=0$ regardless of the value of $G_{i-1}$. Then, due to the structure of OCP, we have $g_{i}\leq 0$, thus $g_{i,\clip}\leq 0$ and $S_{i,\clip}=\lambda_{i-1}S_{i-1,\clip}-g_{i,\clip}\geq \lambda_{i-1}S_{i-1,\clip}$. 
\item If $S_{i-1,\clip}>0$, then $S_{i,\clip}\geq \lambda_{i-1}S_{i-1,\clip}-\abs{g_{i,\clip}}\geq -G_i$. 
\end{itemize}
Combining these two and using an induction from $S_{0,\clip}=0$, we have $S_{t,\clip}\geq -G_t$. 

Summarizing the above completes the proof.
\end{proof}

The following lemma is the full version of Lemma~\ref{lemma:connecting_abridged} presented in Section~\ref{section:conformal}; we further use $V^*_{t,\clip}\leq V^*_{t}$ there.

\begin{lemma}\label{lemma:connecting}
$\A_{CP}$ guarantees for all $t$,
\begin{equation*}
\abs{S^*_t}\leq 2\sqrt{V^*_{t,\clip}}\rpar{1+\sqrt{\log\rpar{1+2r_{t+1}\eps^{-1}}}}+14G^*_t\rpar{1+\sqrt{\log\rpar{1+2r_{t+1}\eps^{-1}}}}^2,
\end{equation*}
where $V^*_{t,\clip}$ and $G^*_t$ are internal quantities of Algorithm~\ref{alg:ocp}.
\end{lemma}

\begin{proof}[Proof of Lemma~\ref{lemma:connecting}]
Given Lemma~\ref{lemma:connecting_basic}, the remaining task is connecting $|S^*_{t,\clip}|$ to $\abs{S^*_{t}}$. To this end, 
\begin{align*}
\abs{S^*_t}\leq\abs{S^*_{t,\clip}}+\sum_{i=1}^t\rpar{\prod_{j=i}^{t-1}\lambda_j}\abs{g^*_i-g^*_{i,\clip}},
\end{align*}
and the sum on the RHS is at most $G^*_t$ following the proof of Theorem~\ref{theorem:base}.
\end{proof}

The next lemma exploits the bounded domain assumption.

\begin{lemma}\label{lemma:bounded_induction}
If $max_t r^*_t\leq D$, then without knowing $D$, $\A_{CP}$ guarantees for all $t$,
\begin{equation*}
\abs{S^*_t}\leq 2\sqrt{V^*_{t,\clip}}\rpar{1+\sqrt{\log\rpar{1+2D\eps^{-1}}}}+15G^*_t\rpar{1+\sqrt{\log\rpar{1+2D\eps^{-1}}}}^2,
\end{equation*}
where $V^*_{t,\clip}$ and $G^*_t$ are internal quantities of Algorithm~\ref{alg:ocp}.
\end{lemma}

\begin{proof}[Proof of Lemma~\ref{lemma:bounded_induction}]
Consider $S^*_{t,\clip}$ defined in Eq.(\ref{eq:discounted_g_sum_clip}). We now use induction to show that
\begin{equation*}
\abs{S^*_{t,\clip}}\leq 2\sqrt{V^*_{t,\clip}}\rpar{1+\sqrt{\log\rpar{1+2D\eps^{-1}}}}+14G^*_t\rpar{1+\sqrt{\log\rpar{1+2D\eps^{-1}}}}^2,
\end{equation*}
and after that, we complete the proof using $\abs{S^*_{t}}\leq \abs{S^*_{t,\clip}}+G^*_t$, just like Lemma~\ref{lemma:connecting}.

Concretely, note that such a statement on $\abs{S^*_{t,\clip}}$ trivially holds for $t=1$. Then, suppose it holds for any $t$. In the $t+1$-th round, there are two cases.
\begin{itemize}
\item \textbf{Case 1: $r_{t+1}\leq D$.}

Due to Lemma~\ref{lemma:connecting_basic}, we have
\begin{align*}
\abs{S^*_{t,\clip}}&\leq 2\sqrt{V^*_{t,\clip}}\rpar{1+\sqrt{\log\rpar{1+2r_{t+1}\eps^{-1}}}}+13G^*_t\rpar{1+\sqrt{\log\rpar{1+2r_{t+1}\eps^{-1}}}}^2\\
&\leq 2\sqrt{V^*_{t,\clip}}\rpar{1+\sqrt{\log\rpar{1+2D\eps^{-1}}}}+13G^*_t\rpar{1+\sqrt{\log\rpar{1+2D\eps^{-1}}}}^2.
\end{align*}
\begin{align*}
\abs{S^*_{t+1,\clip}}&\leq\lambda_t\abs{S^*_{t,\clip}}+\abs{g^*_{t+1,\clip}}\\
&=2\sqrt{\lambda_t^2V^*_{t,\clip}}\rpar{1+\sqrt{\log\rpar{1+2D\eps^{-1}}}}+13\lambda_tG^*_t\rpar{1+\sqrt{\log\rpar{1+2D\eps^{-1}}}}^2+\abs{g^*_{t+1,\clip}}\\
&\leq2\sqrt{V^*_{t+1,\clip}}\rpar{1+\sqrt{\log\rpar{1+2D\eps^{-1}}}}+14G^*_{t+1}\rpar{1+\sqrt{\log\rpar{1+2D\eps^{-1}}}}^2.
\end{align*}

\item \textbf{Case 2: $r_{t+1}> D$.}

In this case, $r_{t+1}>r^*_{t+1}$ so the OCP gradient $g^*_{t+1}\geq 0$, and the clipped gradient $g^*_{t+1,\clip}\geq 0$. Meanwhile, in order to have $r_{t+1}>D\geq 0$, we must have $S^*_{t,\clip}\geq 0$ from the prediction rule. Then, due to the same signs of $g^*_{t+1,\clip}$ and $S^*_{t,\clip}$, we have
\begin{align*}
\abs{S^*_{t+1,\clip}}&=\abs{\lambda_tS^*_{t,\clip}-g^*_{t+1,\clip}}\\
&\leq \lambda_t\abs{S^*_{t,\clip}}\vee\abs{g^*_{t+1,\clip}}\\
&\leq 2\sqrt{\lambda^2_t V^*_{t,\clip}}\rpar{1+\sqrt{\log\rpar{1+2D\eps^{-1}}}}+14\lambda_tG^*_t\rpar{1+\sqrt{\log\rpar{1+2D\eps^{-1}}}}^2\\
&\leq2\sqrt{V^*_{t+1,\clip}}\rpar{1+\sqrt{\log\rpar{1+2D\eps^{-1}}}}+14G^*_{t+1}\rpar{1+\sqrt{\log\rpar{1+2D\eps^{-1}}}}^2.
\end{align*}
\end{itemize}
Combining the two cases, the induction statement holds in the $t+1$-th round. Finally we use $\abs{S^*_{t}}\leq \abs{S^*_{t,\clip}}+G^*_t$. 
\end{proof}

With everything above, Theorem~\ref{theorem:ocp_main} is a simple corollary. 

\ocp*

\begin{proof}[Proof of Theorem~\ref{theorem:ocp_main}]
The regret bound trivially applies. The coverage bound follows from Lemma~\ref{lemma:bounded_induction} and the fact that $V^*_{t,\clip}\leq V^*_{t}$. Then we consider the asymptotic regime of $D\gg \eps$. 
\end{proof}

\section{Detail of experiment}\label{section:more_experiment}

This section presents details of our experiment. Our setup builds on the great work of \cite{bhatnagar2023improved}. 

\paragraph{Setup} We test our OCP algorithm (Algorithm \ref{alg:ocp}, which is based on OCO Algorithm \ref{alg:base} and \ref{alg:base_undiscounted}) in the context of classifying altered images that arrive in a sequential manner. Given a parameterized prediction set $\calC_t(\cdot)$ dependent on the label provided by a base ML model, at each time step our algorithms predict the radius $r_t$ that corresponds to the uncertainty of the base model's prediction, resulting in the prediction set $\calC_t(r_t)$. We adopt the procedure, code, and base model from \cite{bhatnagar2023improved}. Given a sequence of images, we expect that if images are increasingly ``corrupted'' by blur, noise, and other factors, the prediction set size must increase to account for the deviation from the base model's training distribution. We hypothesize that the rate of such a distribution shift also affects the OCP algorithms' performance, thus we test the cases where the corruption levels shift suddenly versus gradually. 

Our algorithms are specifically designed to not use knowledge of the maximum magnitude $D$ of the optimal radius $r^*_t$ (i.e., the maximum uncertainty level). In some prediction scenarios, it is conceivable that $D$ is impossible to know a priori and thus cannot be used as a hyperparameter. In contrast, certain algorithms from the literature use an empirical estimate of $D$ from an offline dataset, which amounts to ``oracle tuning''. These include the Strongly Adaptive Online Conformal Prediction (\textsc{Saocp}) and Scale-Free Online Gradient Descent (\textsc{Sf-Ogd}) proposed by \cite{bhatnagar2023improved}. In our study, we create a modified version of \textsc{Sf-Ogd} called ``Simple OGD'', that does not use such an oracle tuning. The only hyperparameter that we set is the learning rate, which we set to 1 for Simple OGD. Note that despite the name, Simple OGD is also gradient adaptive, which differs from the \textsc{ACI} algorithm from \citep{gibbs2021adaptive}.

\begin{algorithm}[t]
\caption{Modified 1D magnitude learner on $[0,\infty)$.\label{alg:zero_ht}}
\begin{algorithmic}[1]
\REQUIRE Hyperparameter $\eps>0$. 
\STATE Initialize parameters $v_1>0$, $s_1=0$.
\FOR{$t=1,2,\ldots$}
\STATE Define the unprojected prediction $\tilde{x}_t$,
\begin{equation*}
\tilde x_t=\eps\cdot\erfi\rpar{\frac{s_{t}}{2\sqrt{v_{t}}}}.
\end{equation*}
\STATE Predict $x_t=\Pi_{[0,\infty)}\rpar{\tilde x_t}$, the projection of $\tilde x_t$ to the domain $[0,\infty)$.
\STATE Receive the 1D loss gradient $g_t\in\R$ and the discount factor $\lambda_{t-1}\in(0,\infty)$.
\STATE If $g_{t}\tilde x_t<g_{t}x_t$, define a surrogate loss gradient $\tilde g_{t}=0$. Otherwise, $\tilde g_{t}=g_{t}$.
\STATE Update $v_{t+1}=\lambda_{t-1}^2v_{t}+\tilde g^2_{t}$, $s_{t+1}=\lambda_{t-1}s_{t}-\tilde g_{t}$.
\ENDFOR
\end{algorithmic}
\end{algorithm}

\paragraph{Baselines} We perform OCP for ten different algorithms:

\begin{enumerate}
    \item \textsc{MagL-D}: We test our Algorithm \ref{alg:base} with $\varepsilon = 1$ and discount factor $\lambda_t = 0.999$, which we name \textsc{MagL-D} (\textbf{Mag}nitude Learner with \textbf{L}ipschitz Constant Estimate and \textbf{D}iscounting).
    \item \textsc{MagL}: We test Algorithm \ref{alg:base_undiscounted}, the undiscounted algorithm that uses the running estimate of the Lipschitz constant, $h_t$, with $\varepsilon = 1$. We name this algorithm \textsc{MagL} (\textbf{Mag}nitude Learner with \textbf{L}ipschitz constant estimate).
    \item \textsc{MagDis}: We also test Algorithm \ref{alg:zero_ht} with $\varepsilon = 1$ and $\lambda_t=0.999$, which is a simplified version of Algorithm \ref{alg:base} that essentially sets $h_t = 0$, does not clip $g_t$, and initializes $v_t > 0$. We name this algorithm \textsc{MagDis} (\textbf{Mag}nitude Learner with \textbf{Dis}counting).
\end{enumerate}  

Using the implementation from \citep{bhatnagar2023improved}, we also obtain results for \textsc{Saocp}, \textsc{Sf-Ogd}, Simple OGD, Standard Split Conformal Prediction (SCP) \cite{vovk2005algorithmic}, Non-Exchangeable SCP (NExCP) \cite{barber2023conformal}, Fully-Adaptive Conformal Inference (\textsc{Faci}) \cite{gibbs2022conformal}, and  \textsc{Faci-S} \cite{bhatnagar2023improved}.

\paragraph{Metrics} We choose a targeted coverage rate of 90\% for all experiments, which means $\alpha=0.1$. To quantify the algorithms' performance, we follow the definitions from \cite{bhatnagar2023improved} to evaluate the coverage (local and average), prediction width (local and average), worst-case local coverage error ($\mathrm{LCE}_k$), and runtime. They are defined as follows. 
\begin{itemize}
    \item First, let $Y_t$ be the true label of the $t$-th image, and for brevity, let $\widehat C_t\leftarrow \calC_t(r_t)$ be the $t$-th prediction set over the labels. Then, $\mathrm{err}_t$ is the indicator function of miscoverage at time $t$:
    \begin{equation*}
       \mathrm{err}_t\defeq \bm{1} \left[Y_t \notin \widehat{C}_t \right],
    \end{equation*}
    \noindent where $\mathrm{err}_t =1$ if the prediction set $\widehat{C}_t$ does not include the true label $Y_t$.
    
    \item \textbf{Local Coverage}\quad Over any sliding window, $k$, the local coverage is defined as:
    \begin{equation*}
        \mathrm{Local\,Coverage}(t)\defeq \frac{1}{k} \sum_{i = t}^{t+k-1} \rpar{1-\mathrm{err}_i}.
    \end{equation*}
    For all trials, we used an interval length of $k=100$. 
    \item \textbf{Average Coverage}\quad The average coverage is similarly defined, but averaged over the total time steps $T$. For all experiments, $T = 6011$.
    \begin{equation*}
        \mathrm{Avg. \, Coverage} := \frac{1}{T} \sum^T_{i=0} \left( 1 - \mathrm{err}_i \right).
    \end{equation*}
    \item \textbf{Local Width}\quad The local width is the cardinality of $\widehat{C}_t$, averaged over the length $k$ time window:
    \begin{equation*}
        \mathrm{Local \, Width} (t)\defeq \frac{1}{k} \sum_{i = t}^{t+k-1} \left| \widehat{C}_t\right|.
    \end{equation*}

    It is compared to the ``best fixed'' local width defined as follows. If we let $C^*_t\leftarrow \calC_t(r^*_t)$ denote the optimal prediction set had we known the optimal radius $r^*_t$ beforehand, then the \emph{best fixed local width} $\mathrm{Local \, Width}^* (t)$ is defined as the $1-\alpha$ quantile of $\left\{\abs{C^*_i}; i\in[t:t+k-1]\right\}$. 

    \item \textbf{Average Width}\quad Similarly, the average width is defined as:
    \begin{equation*}
        \mathrm{Avg. Width}:= \frac{1}{T}  \sum^{T}_{t=1} \left| \widehat{C}_t\right|.
    \end{equation*}
    \item \textbf{Local Coverage Error (LCE)}\quad The LCE over the sliding window of length $k$ is defined as:
    \begin{equation*}
        \mathrm{LCE}_k := \max_{\tau, \tau+k-1 \subseteq [1,T]} \left| \alpha - \frac{1}{k} \sum_{t=\tau}^{\tau+k-1} \mathrm{err}_t \right|.
    \end{equation*}
    Essentially, it means the largest deviation of the empirical miscoverage rate (evaluated over sliding time windows of length $k$) from the targeted miscoverage rate $\alpha$.
\end{itemize} 

\paragraph{Main results}

Our main results are shown in Figure \ref{fig:TinyImageNet_1D} and Table \ref{tab:methods_performance} (Section~\ref{section:experiment}). The purpose of Figure \ref{fig:TinyImageNet_1D} is to demonstrate the dependence of the local coverage and the local width on (1) sudden and (2) gradual distribution shifts (i.e., the time-varying corruption level).\footnote{For better visibility, a 1D Gaussian filter is applied to the local width and the local coverage plots, same as in \cite{bhatnagar2023improved}.} The purpose of Table \ref{tab:methods_performance} is to summarize the performance of our algorithms and compare them to the baselines; the results there correspond to the case of sudden distribution shift.

\begin{figure}[ht]
    \centering
    \includegraphics[width=\textwidth]{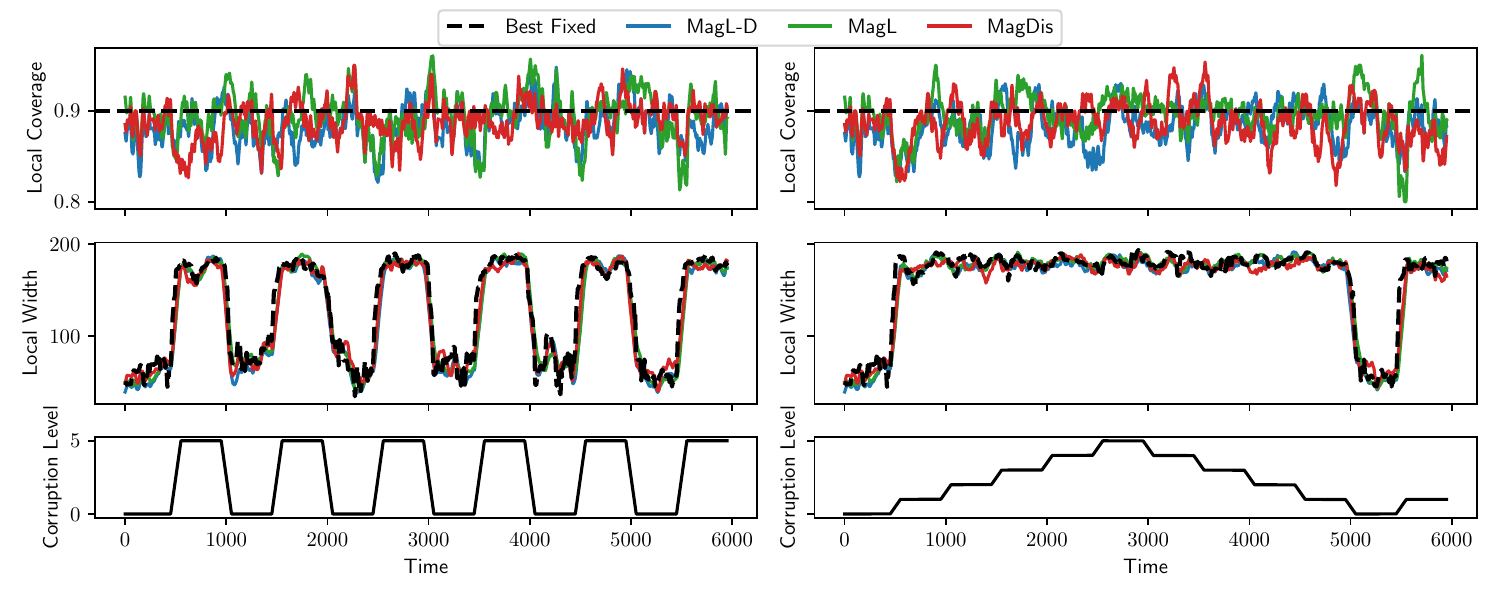}
    \caption{The local coverage (first row), local width (second row), and corresponding corruption level (third row) of our algorithms. Results are obtained using corrupted versions of TinyImageNet, with time-varying corruption level (distribution shift). (Left) Results for \textit{sudden} changes in corruption level. (Right) Results for \textit{gradual} changes in corruption level. The distribution shifts every 500 steps. Moving averages are plotted with a window size of 100 time steps ($k=100$).}
    \label{fig:TinyImageNet_1D}
\end{figure}

Figure \ref{fig:TinyImageNet_1D} justifies the validity of our algorithms. Specifically, the local coverage of our algorithms fluctuate around the target coverage of 0.9 for both sudden and gradual distribution shifts. Similarly, the local width approximately replicates the best fixed local width, $\mathrm{Local \, Width}^* (t)$, as shown in Figure \ref{fig:TinyImageNet_1D} (middle row).

\paragraph{Hyperparameter sensitivity} We also test the algorithms' sensitivities to the offline estimate of $D$, which we denote as $D_{\mathrm{est}}$. The baselines \textsc{Saocp} and \textsc{Sf-Ogd} require this parameter to initialize, while Simple OGD and all three of our algorithms do not. To this end, we rerun the experiments above and change the ratio of $D_{\mathrm{est}}$ to the true value $D$. The following settings are tested:
\begin{equation*}
    \frac{D_{\mathrm{est}}}{D} = 
    \begin{Bmatrix}
    10^{-3}, & 10^{-2}, & 10^{-1}, & 10^{0}, & 10^{1}, & 10^{2}, & 10^{3}
    \end{Bmatrix}.
\end{equation*}
We plot the influence of $D_{\mathrm{est}}$ on the average coverage and the average width in Figure \ref{fig:figure_summary}. For these experiments, we use the case when image corruptions are sudden. Coverage being much larger or much less than 0.9 is not a desirable behavior; given satisfactory coverage, lower width is desirable. 

\begin{figure}[ht]
    \centering
    \subfloat{
    \includegraphics[width=\textwidth]{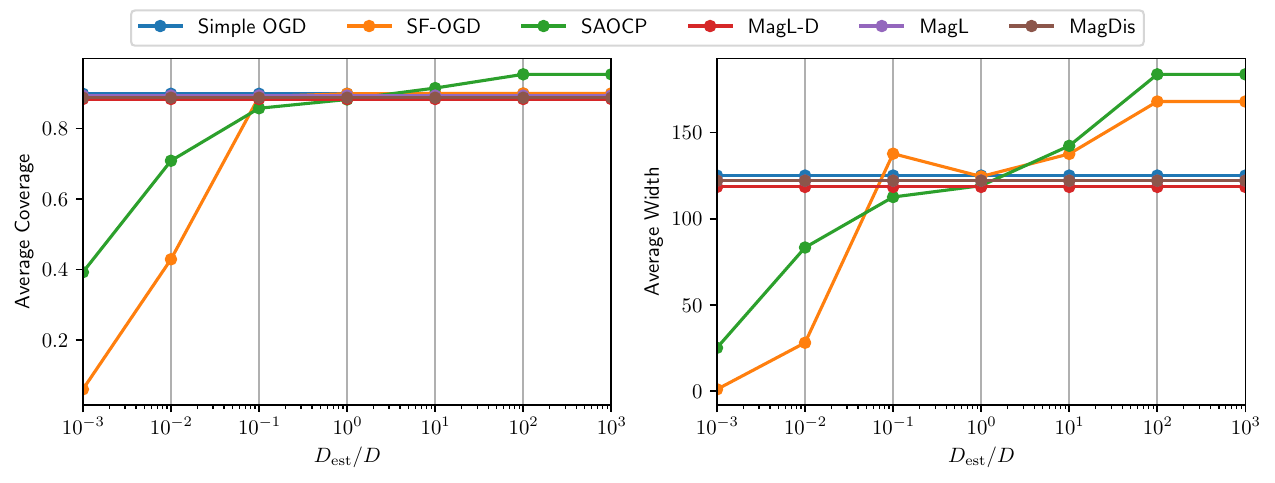}
    }
    \hfill
    \caption{The average coverage (first row) and average width (second row) as a function of the estimated maximum radius, $D_{\mathrm{est}}$ relative to the true radius $D$. The performance of \textsc{Sf-Ogd} and \textsc{Saocp} \cite{bhatnagar2023improved} are sensitive to $D_{\mathrm{est}}/D$. Averages are taken over the entire time horizon, where the total time steps $T = 6011$.}
    \label{fig:figure_summary}
\end{figure}

As $D_{\mathrm{est}}/D$ increases, both the average coverage and average width increase for \textsc{Saocp} and \textsc{Sf-Ogd}. In the opposite direction, the average coverage and average width also decrease as $D_{\mathrm{est}}/D$ decrease for \textsc{Saocp} and \textsc{Sf-Ogd}. When $D_{\mathrm{est}} = D$, Simple OGD and \textsc{Sf-Ogd} have approximately equivalent performance. This is not to say that using an offline estimate of $D_{\mathrm{est}}/D$ does not improve performance. The catch is that, for the particular experimental dataset, setting the learning rate to 1 in Simple OGD coincidentally just happens to be a well-picked learning rate, given the true value of the maximum radius magnitude $D$. In contrast, our magnitude learners remain fixed under variations in $D_{\mathrm{est}}/D$, since they do not need $D_{\mathrm{est}}$ to initialize. The baselines SCP, NExCP, \textsc{Faci}, and \textsc{Faci-S} are not presented in Figure \ref{fig:figure_summary} as to better highlight the performance of \textsc{Saocp}, OGD, and our magnitude learners. Benchmarks on these additional baselines can be found in \cite{bhatnagar2023improved}.

\begin{figure}[ht]
    \centering
    \includegraphics[width=0.7\linewidth]{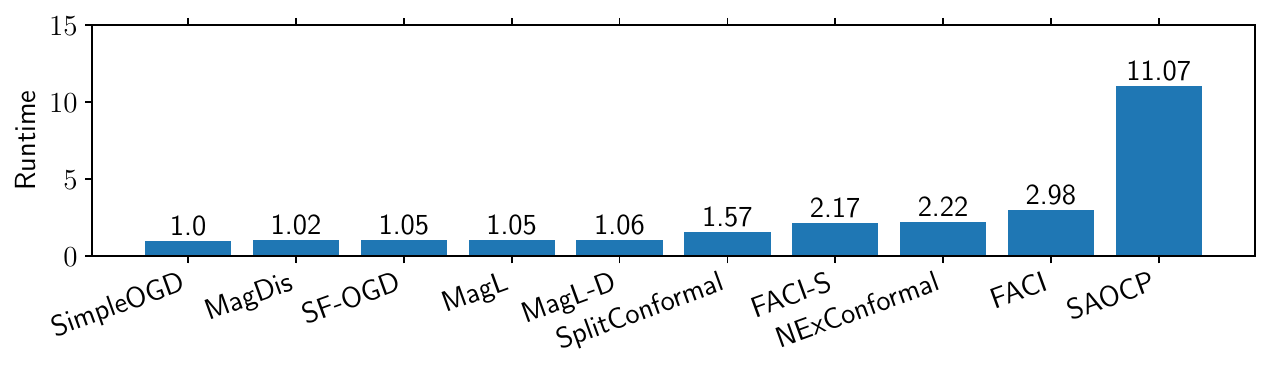}
    \caption{The runtime per time step of each algorithm, normalized to the runtime of \textit{Simple OGD}. The runtime of \textsc{Saocp} is longest due to it being a meta-algorithm that initializes $\textsc{Sf-Ogd}$ on each time step.}
    \label{fig:run_time}
\end{figure}

\paragraph{Runtime} We also measure the time for each algorithm to complete a single prediction. Runtime results are provided in Figure \ref{fig:run_time}. The results are normalized relative to the runtime of Simple OGD. Note the runtime standard deviations in Table \ref{tab:methods_performance}. Accounting for standard deviations, the runtime differences between Simple OGD, \textsc{Sf-Ogd}, \textsc{MagL}, and \textsc{MagDis} are negligible. The runtime of \textsc{Saocp} is longest due to it being a meta-algorithm that initializes $\textsc{Sf-Ogd}$ on each time step, cf., Appendix~\ref{section:more_related}. 

\end{document}